%% file: paper_v7.16.tex
\documentclass[twoside, 11pt]{article}

\usepackage{graphicx, subcaption}
\graphicspath{{figures/}}
\usepackage[top=1in,bottom=1in,left=1in,right=1in,footnotesep=0.3in]{geometry}
\usepackage[sort&compress, numbers]{natbib} 
\usepackage{amsfonts, amsmath, amssymb, amsthm}
\usepackage{MnSymbol}
\usepackage[shortlabels]{enumitem}
\usepackage{booktabs}
\usepackage{algorithm, algorithmic}

\usepackage{color}

\usepackage{hyperref}
\hypersetup{colorlinks=true}
\usepackage{url}

\input notation.tex

\usepackage{tikz}
\usepackage{tikz-3dplot}
\usepackage{pgfplots}
\pgfplotsset{compat=1.16}
\usepgfplotslibrary{fillbetween}
\usepgfplotslibrary{patchplots}

\usepackage{float}
\usepackage{mathtools}
\mathtoolsset{showonlyrefs}

\newcommand\restr[2]{{
  \left.\kern-\nulldelimiterspace
  #1 
  \vphantom{\big|} 
  \right|_{#2} 
}}

\begin{document}

\title{Minimax Estimation of Distances on a Surface and \\ Minimax Manifold Learning in the Isometric-to-Convex Setting}
\author{
Ery Arias-Castro\,\footnote{Department of Mathematics and Halıcıoğlu Data Science Institute, University of California, San Diego, USA} 
\and  
Phong Alain Chau\,\footnote{Department of Mathematics, University of California, San Diego, USA}
}
\date{}
\maketitle


\begin{abstract} 
We start by considering the problem of estimating intrinsic distances on a smooth submanifold. We show that minimax optimality can be obtained via a reconstruction of the surface, and discuss the use of a particular mesh construction --- the tangential Delaunay complex --- for that purpose. 
We then turn to manifold learning and argue that a variant of Isomap where the distances are instead computed on a reconstructed surface is minimax optimal for the isometric variant of the problem.
\end{abstract}


\section{Introduction}
The estimation of shortest paths and intrinsic distances on surfaces is a fundamental problem in
computational geometry with wide-ranging applications.
In motion planning, shortest paths represent resource-efficient
sequences of actions to be undertaken by the agent in some given configuration space~\cite{lavalle2006planning,latombe2012robot}.
In addition to the clear applications to robot locomotion and manipulation, this framework has bore fruit in the field of
biology wherein proteins and folding networks are of great interest~\cite{Amato2002, Thomas2005}.
In cluster analysis, geodesic distances have found use as a similarity metric
to create partitions that respect the underlying geometry~\cite{kaufmann1987clustering, park2009simple, li2019geodesic}.
In manifold learning (aka nonlinear dimensionality reduction), the Isometric Feature Mapping (Isomap) algorithm crucially depends on the
approximation of geodesic distances on the underlying surface~\cite{tenenbaum2000global}, and so does another important algorithm, Maximum Variance Unfolding (MVU)~\cite{mvu04} --- although in disguise~\cite{paprotny2012connection, arias2013convergence}.
More generally, the estimation of distances is at the core of some important methods for embedding a graph (aka multidimensional scaling)~\cite{kruskal1980designing, shang2003localization, shang2004improved, niculescu2003dv}.

\subsection{Existing error bounds}
\label{sec:existing bounds}
Consider a set of points $\bX = \{x_1, \dots, x_n\}$ in some Euclidean space assumed to belong to some unknown $C^2$ submanifold $\cM$. The goal is to estimate their pairwise (intrinsic) distances on $\cM$ and possibly provide corresponding shortest paths. 
Therefore, if $\d_\cM$ denotes the intrinsic distance on $\cM$, then the goal is to estimate $\d_{\cM}(x_i, x_j)$ for all $i, j \in [n] := \{1, \dots, n\}$.

For this goal to be achievable in a nonparametric setting where not much is known about $\cM$ except being smooth (see \aspref{M} for details) requires that the point set be sufficiently dense in $\cM$. To quantify that, suppose\footnote{The use of intrinsic distances could be used instead, but this would not change things in any noticeable way.}
\begin{equation} \label{eps}
\max_{x \in \cM} \min_{i = 1, \dots, n} \|x-x_i\| \le \eps.
\end{equation}
Note that $\eps$ is at best on the order of $(\log(n)/n)^{1/k}$ when the points are sampled uniformly at random from $\cM$ and $\cM$ is of dimension $k$.
Throughout, we assume that $k$ is known, although this is non-essential as it can be reliably estimated~\cite{fukunaga1971algorithm, kim2019minimax}.

The first error bounds we know of come from the literature on manifold learning. Indeed, Bernstein {\it et al} provide some theory for Isomap in~\cite{bernstein2000graph}. Isomap is based on three main steps: 1) form a neighborhood graph where the nodes are the points and two points within distance $r$ are connected with an edge weighted by the Euclidean distance between the points; 2) compute all the pairwise graph distances; 3) apply Classical Scaling to these distances with a prescribed embedding dimension $k$. The connectivity radius $r$ is a tuning parameter of the method. Bernstein {\it et al} focus on the first two steps, meaning on the estimation of the intrinsic distances. Let $\d_\cG$ denote the graph metric, and note that it depends on $r$. Bernstein {\it et al} are able to show that, if $\cM$ is geodesically convex and $\eps/r \le C_1$, then
\begin{equation} 
\label{bernstein_bound} 
(1 - C_2 \eps/r) \d_\cG(x_i, x_j) \le \d_\cM(x_i, x_j) \le (1 + C_2 r^2) \d_\cG(x_i, x_j), \quad \forall i, j \in [n],
\end{equation}
where $C_1, C_2$ are constants depending on $\cM$.

The assumption of geodesic convexity is in fact not needed for~\eqref{bernstein_bound} to hold as long as the shortest paths on $\cM$ have curvature bounded by some $C$ depending on $\cM$, as shown in~\cite{arias2019unconstrained}. In that paper, the upper bound is derived based on the seminal work of Dubins~\cite{dubins1957curves} (the lower bound can be obtained by elementary means), and the problem is also considered under a curvature constraint on the paths.
The lower bound~\eqref{bernstein_bound} is derived independently by Oh {\it et al}~\cite{oh2010sensor} in the context of a convex domain, motivated by the problem of placing sensors that are only aware of other sensors within a prescribed distance --- one variant of the sensor network localization problem. Note that the upper bound $\d_\cM(x_i, x_j) \le \d_\cG(x_i, x_j)$ holds in that case.  
In the same setting, Janson {\it et al}~\cite{janson2018deterministic} derive a similar lower bound in the context of path planning in robotics in the presence of obstacles (although with some clearance) and where again the upper bound is trivial.
Arias-Castro {\it et al}~\cite{arias2020perturbation} sharpen the lower bound, replacing $\eps/r$ with $(\eps/r)^2$. They do so in the more general setting where $\cM$ is isometric to a convex domain. 

In summary, for general $C^2$ submanifolds, the best available bound remains~\eqref{bernstein_bound} as established recently in~\cite{arias2019unconstrained}. And if one optimizes the bounds in terms of $r$ (the tuning parameter here) --- a task that in principle requires knowledge of $\eps$ --- we find that the relative error rate is in $O(\eps^{2/3})$, specifically,
\begin{equation} 
\label{available_bound} 
|\d_\cM(x_i, x_j) - \d_\cG(x_i, x_j)| \le C \eps^{2/3} \d_\cM(x_i, x_j), \quad \forall i, j \in \{1, \dots, n\},
\end{equation}
where $C$ is a constant that depends on $\cM$.
In fact, this rate already appears in \cite{aaron2018convergence}.
If $\cM$ is isometric to a convex domain, the improved result in~\cite{arias2020perturbation} leads to a relative error rate in $O(\eps)$.

\begin{remark} 
Note that in sensor network localization~\cite{oh2010sensor} and in path planning~\cite{janson2018deterministic}, the quantity $r$ is typically {\em not} a parameter that the user can change. We are here in the original context of points in space, as in manifold learning.
\end{remark}

\subsection{A new error bound}
It turns out that~\eqref{available_bound} is far from optimal. Indeed, we show that it is possible to obtain estimates $\hat d_{ij}$ such that
\begin{equation} 
\label{distance_error} 
|\d_\cM(x_i, x_j) - \hat d_{ij}| \le C \eps^2 \min\{\d_\cM(x_i, x_j), \hat d_{ij}\}, \quad \forall i, j \in [n],
\end{equation}
where $C$ is again a generic constant depending on $\cM$.

We first propose a non-constructive approach that consists in interpolating the data points by a smooth surface, and then estimating the distance on $\cM$ by the distance on that interpolating surface. 

We then propose a more practical approach based instead on a mesh construction. The particular mesh construction that we use is the tangential Delaunay complex~\cite{Boissonnat2014, Boissonnat2018, Boissonnat2004, Freedman2002}. In fact, because it requires knowledge of the tangent subspaces to the surface $\cM$ at the sample points, we follow Aamari and Levrard~\cite{aamari2018stability} and first estimate the tangent spaces.
See \figref{basic_illustration} for an illustration.

\begin{figure}[tbp] 
\centering 
\includegraphics[width=0.40\linewidth]{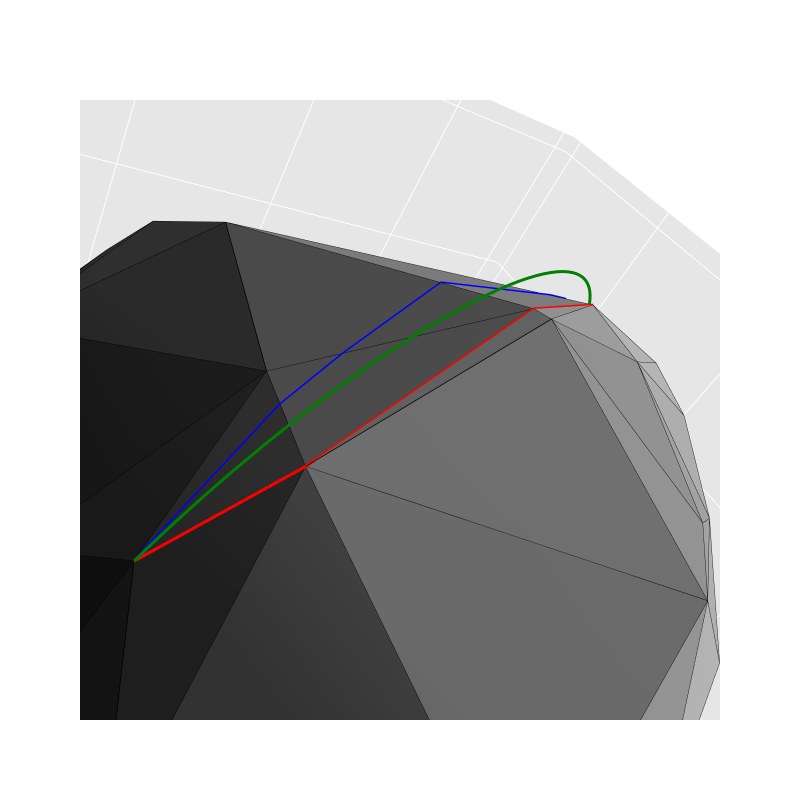} 
\includegraphics[width=0.40\linewidth]{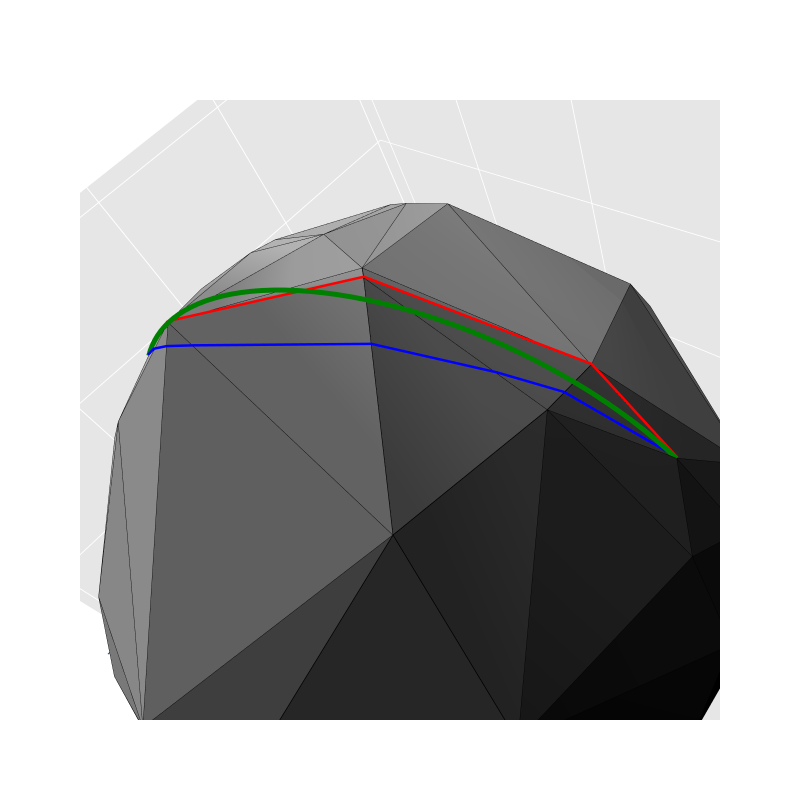} 
\caption{A setting where the underlying surface is a sphere. The shortest path on the surface (green), on the mesh (blue), and on the neighborhood graph (red), between two sample points are shown.} 
\label{fig:basic_illustration}
\end{figure}

In addition to proposing estimators that satisfy the performance bound~\eqref{distance_error}, we show that the relative error rate in $O(\eps^2)$ that results from that bound is best possible in an information-theoretic sense ---w even if we know that $\cM$ is isometric to a convex domain.

\begin{remark} 
Experts in computational geometry are aware of approximating meshes providing an approximation to the metric on the surface. For example, the paper~\cite{dyer2015riemannian} provides sufficient conditions for a mesh construction to satisfy an approximation bound like our \thmref{approximation}.\footnote{\cite[Th~3]{dyer2015riemannian}, as stated, provides a distortion bound in $O(1)$, which is of course less than satisfactory. However, a quick inspection reveals that, with the notation of that paper, one can take $h \le h_0 := \min\{\iota_M/4, t_0/6\sqrt{\Lambda}\}$ instead of $h = h_0$ as stated in the result, leading to a bound in $O(h^2)$, which is of similar order as~\eqref{distance_error} as $h$ there plays the role of $\eps$ here. We note in passing that the sufficient conditions provided in that theorem for a mesh to provide a good metric approximation rely on knowing the underlying surface.} 
This result is used in~\cite{boissonnat2018delaunay} to derive a bound similar to~\eqref{distance_error} for a mesh construction based on knowledge of an atlas of the underlying surface $\cM$. Here we show that knowledge of underlying surface is in fact not needed --- an important difference as we adopt an estimation/information-theoretic stance.
\end{remark}

\begin{remark}
After our work was made public, Aamari, Berenfeld and Levrard \cite{aamari2022optimal} obtained minimax bounds on the estimation of the metric of a $C^q$ submanifold based on a sample drawn iid from a density supported on the submanifold.
Although the setting is a little different, the bound is also in $O(\eps^2)$ when $q = 2$.
\end{remark}

\subsection{Application: minimax manifold learning}
We already mentioned one of the main methods for manifold learning, Isomap~\cite{tenenbaum2000global}, which consists in estimating the pairwise intrinsic distances by shortest path distances in a neighborhood graph, followed by an application of Classical Scaling. 
Arias-Castro {\it et al}~\cite{arias2020perturbation} derive an error bound for Isomap based on a perturbation bound for Classical Scaling. 

An improved estimation of the pairwise intrinsic distances naturally leads to an improved performance. We show that the resulting performance bound --- obtained by a combination of the new bound \eqref{distance_error} and perturbation bounds available in~\cite{arias2020perturbation} --- is optimal in an information-theoretic sense for the problem of manifold learning in the setting where the submanifold $\cM$ is isometric to a convex domain. 

In its more practical form, where a mesh reconstruction of the surface is used, we call the resulting method {\em Mesh Isomap}. This idea of using a mesh to improve on Isomap is not entirely and is brought up in a discussion~\cite{balasubramanian2002isomap} of the main Isomap paper~\cite{tenenbaum2000global}. We elaborate on this point in \remref{isomap_discussion}.


\subsection{Content}
The rest of the paper is organized as follows.
In \secref{preliminaries}, we list or quickly derive some results that will prove useful later on in the paper.
In \secref{distortion}, we derive an estimator that satisfies the announced performance bound~\eqref{distance_error}. We also show that this cannot be improved upon from an information-theoretic perspective. 
While the estimator defined and studied in that section is not constructive, in \secref{meshes} we propose a more practical alternative based on a particular mesh construction --- the tangential Delaunay complex --- which we show achieves the same level of performance.
We show the method in action in some numerical experiments.
In \secref{isomap}, we turn to the problem of (isometric) manifold learning and apply the conclusions of the previous sections to derive a minimax optimal procedure and showcase Mesh Isomap in some numerical experiments.
\secref{discussion} is a brief discussion section.

\section{Preliminaries}
\label{sec:preliminaries}

In this section we introduce some concepts and tools that will be used in subsequent sections to derive the main results.

\subsection{Length of a curve}
Before proceeding, we recall that a curve in $\bbR^d$ can be defined as the range of a continuous function $\gamma: [0, 1] \to \bbR^d$. We may identify a curve $\gamma$ with one of its parameterization without warning.
The length of a curve $\gamma$ is defined as
\begin{equation} 
\label{length} 
\Lambda(\gamma) := \sup \sum_j \|\gamma(t_{j+1}) - \gamma(t_j)\|,
\end{equation}
where the supremum is over all increasing sequences $(t_j) \subset [0,1]$.
We also note that, if $\cM$ is a closed topological submanifold of a Euclidean space without boundary,\footnote{This properties is also satisfied by topological submanifolds with boundary under some conditions on the boundary. We focus on submanifolds without boundary as these are the objects that occupy us in the present paper.}
then for each pair of points $x, x' \in \cM$ there is a shortest path on $\cM$ joining them, meaning that the following infimum is attained
\begin{equation} 
\inf\big\{\Lambda(\gamma) \mid \gamma: [0,1] \to \cM, \gamma(0) = x, \gamma(1) = x'\big\}.
\end{equation}

The following is by definition of parameterization by arc length.
\begin{lemma}\label{lem:arclength}
Consider a differentiable and injective function $\gamma: [0,1] \to \bbR^d$. 
For $t \in [0,1]$, let $\lambda(t)$ denote the length of $\gamma([0,t]) = \{\gamma(s): 0 \le s \le t\}$. 
Then there is a differentiable and injective function $\nu : [0, \Lambda(\gamma)] \to \bbR^d$ differentiable such that $\nu(\lambda(t)) = \gamma(t)$ for all $t \in [0,1]$. Since satisfies $\|\dot\nu(s)\| = 1$ for all $s$, $\nu$ is an isometric diffeomorphism between $[0, \Lambda(\gamma)]$ and $\gamma([0,1])$.
\end{lemma}

\subsection{Distortion maps}
A map $F: U \to \bbR^d$, where $U\subset\bbR^p$ where $p$ and $d$ may be different, is called a $\xi$-distortion map if
\begin{equation} 
\label{distortion_map} 
\big| \|F(x) - F(y) \| - \|x - y \| \big| \leq \xi \|x - y\|, \quad \forall x, y \in U.
\end{equation}
Note that a $\xi$-distortion map is Lipschitz with constant $(1+\xi)$, and if $\xi < 1$, it is injective and its inverse (defined on its range) is also Lipschitz with constant $(1-\xi)^{-1}$, and in fact, a $\xi/(1-\xi)$-distortion map. This is important because of the following.

\begin{lemma} 
\label{lem:distortion} 
For any curve $\gamma$ and any $L$-Lipschitz function $F$, $\Lambda(F(\gamma))\leq L\, \Lambda(\gamma)$. As a consequence, if $F$ is a $\xi$-distortion map with $\xi < 1$, then for any curve $\gamma$, 
\begin{equation} 
    \label{distortion} 
    \big|\Lambda(\gamma) - \Lambda(F(\gamma))\big| \le \frac\xi{1-\xi}\, \min\{\Lambda(\gamma), \Lambda(F(\gamma))\}. 
\end{equation}
\end{lemma}

\begin{proof} 
We provide a proof for completeness. For the first part, we first note that $F(\gamma) = F\circ\gamma$ is indeed a curve. Also, for an increasing sequence $(t_j) \subset [0,1]$, we have 
\begin{align} 
    \sum_j \|F\circ\gamma(t_{j+1}) - F\circ\gamma(t_j)\| 
     & \le \sum_j L \|\gamma(t_{j+1}) - \gamma(t_j)\|, 
\end{align} 
and taking the supremum over all such sequences leads to the desired bound.
 
For the second part, meaning \eqref{distortion}, from the first part we obtain $\Lambda(F(\gamma)) \le (1+\xi) \Lambda(\gamma)$ since $F$ is $(1+\xi)$-Lipschitz. Let $\eta = F(\gamma)$. Then $F^{-1}$ is obviously defined on $\eta$, and being $(1-\xi)^{-1}$-Lipschitz, the first part gives $\Lambda(F^{-1}(\eta)) \le (1-\xi)^{-1} \Lambda(\eta)$, or equivalently, $\Lambda(\gamma) \le (1-\xi)^{-1} \Lambda(F(\gamma))$. 
This gives the first inequality in~\eqref{distortion}. 
Tow cases are possible.
If $\Lambda(\gamma) \ge \Lambda(F(\gamma))$, then we use the first inequality to get 
\[
0 \le \Lambda(\gamma) - \Lambda(F(\gamma))\le (1+\xi) \Lambda(\gamma) - \Lambda(\gamma) 
= \xi \Lambda(\gamma).\]
If $\Lambda(\gamma) \le \Lambda(F(\gamma))$, then we use the second inequality to get 
\[
0 \le \Lambda(F(\gamma)) - \Lambda(\gamma)\le (1-\xi)^{-1} \Lambda(\gamma) - \Lambda(\gamma) 
= \frac\xi{1-\xi} \Lambda(\gamma).\]
In either case, \eqref{distortion} is implied.
\end{proof}

The following is a simple corollary of this lemma. Although straightforward, it is at the very root of this idea of using a surface reconstruction to obtain better approximating rates for the intrinsic distances. In what follows, $\cS$ should be thought of as playing the role of approximating surface to $\cM$, even as they in fact play symmetric roles.

\begin{corollary} 
\label{cor:distortion} 
Suppose that $\cS \subset \bbR^p$ is a closed topological submanifold without boundary and that $F: \cS \to \bbR^d$ is some $\xi$-distortion map with $\xi < 1$. Then $\cM := F(\cS) \subset \bbR^d$ is also a closed topological submanifold without boundary. Morevover, the distance on $\cM$ can be approximated by the distance on $\cS$ to within a relative error of $(1-\xi)^{-1}$ in the sense that 
\begin{align} 
    \big| \d_\cM(x, x') - \d_{\cS}(F^{-1}(x), F^{-1}(x')) \big|
    & \le \frac\xi{1-\xi} \min\{ \d_\cM(x, x'), \d_{\cS}(F^{-1}(x), F^{-1}(x'))\},
     \quad \forall x, x' \in \cM. 
\end{align}
\end{corollary}

\begin{proof} 
The fact that $\cM$ is a closed topological submanifold without boundary is because $\cS$ satisfies these properties by assumption and $F$ is a homeomorphism between $\cS$ and $\cM$. 
Now, take $x, x' \in \cM$, and let $\gamma$ be a shortest path on $\cS$ between $y := F^{-1}(x)$ and $y' := F^{-1}(x')$ so that $\d_{\cS}(y,y') = \Lambda(\gamma)$. 
Applying \lemref{distortion}, using the fact that $F$ is $(1+\xi)$-Lipschitz, we have that $\Lambda(F(\gamma)) \le (1+\xi)\Lambda(\gamma)$. And since $F(\gamma) = F\circ\gamma$ is a curve on $\cM$ between $x$ and $x'$, we have $\d_\cM(x,x') \le  \Lambda(F(\gamma))$. We thus have 
\[\d_\cM(x,x') 
    \le \Lambda(F(\gamma)) 
    \le (1+\xi)\Lambda(\gamma) 
    = (1+\xi) \d_{\cS}(y,y').\] 
Similarly, using the fact that $F^{-1}$ is $(1-\xi)^{-1}$-Lipschitz, we obtain 
\[\d_{\cS}(y,y') \le (1-\xi)^{-1} \d_\cM(x,x').\] 
We conclude combining these two bounds.
\end{proof}

\subsection{Medial axis, reach, and metric projection}
The medial axis of $\cM$, denoted ${\rm ax}(\cM)$, is the set of points in $\bbR^d$ that have two or more closest points on $\cM$. We define the (metric) projection onto $\cM$ as $P_\cM: \bbR^d \setminus {\rm ax}(\cM) \to \cM$ that sends a point $x$ to its (unique) closest point on $\cM$.
The reach of $\cM$ the infimum of the distance between a point in $\cM$ and ${\rm ax}(\cM)$~\cite{Federer1959}.
It is well-known that a compact connected $C^2$ submanifold without boundary has a (strictly) positive reach, and that the inverse of the reach bounds from above the (sectional) curvature on $\cM$, pointwise.

Recall that the $h$-tubular neighborhood of $\cM \subset \bbR^d$ is the set of all points that are within distance $h$ of $\cM$, meaning $\{x: \dist(x, \cM) \le h\}$.

\begin{lemma}[Th~4.8(8) in~\cite{Federer1959}; Lem~7.13 in~\cite{Boissonnat2018}] 
\label{lem:pi_lipschitz} 
If $\cM$ has reach $\ge \rho$, then for any $h < \rho$, $P_\cM$ is $\rho/(\rho-h)$-Lipschitz on the $h$-tubular neighborhood of $\cM$.
\end{lemma}

\begin{lemma} 
\label{lem:both_projections} 
If $\cM$ and $\cS$ have reach $\ge \rho$ and are within Hausdorff distance $h \le \rho/2$, then 
\begin{equation} 
    |\d_\cM(x,x') - \d_{\cS}(x, x')| 
    \le (2h/\rho) \min\{\d_\cM(x,x'), \d_{\cS}(x,x')\}, \quad \forall x, x' \in \cM \cap \cS. 
\end{equation}
\end{lemma}

\begin{proof} 
Let $\gamma$ be a shortest path on $\cM$ between $x$ and $x'$, so that $\Lambda(\gamma) = \d_\cM(x,x')$. 
Because $\gamma \subset \cM$, $\gamma$ is entirely in the $h$-tubular neighborhood of $\cS$, and we may define $\zeta := P_{\cS}(\gamma)$, which is a curve on $\cS$ joining $x$ and $x'$. In particular, $\d_{\cS}(x,x') \le \Lambda(\zeta)$. 
The fact that $\gamma$ is entirely in the $h$-tubular neighborhood of $\cS$ also implies via \lemref{pi_lipschitz} that $P_{\cS}$ is $(1+\xi)$-Lipschitz on $\gamma$ with $\xi := h/(\rho-h)$, in turn implying via \lemref{distortion} that $\Lambda(\zeta) \le (1+\xi)\, \Lambda(\gamma)$. We have thus established that $\d_{\cS}(x,x') \le (1+\xi)\, \d_\cM(x,x')$. 

The reverse inequality holds by symmetry, given that $\cM$ and $\cS$ play the same role, and applying these two bounds together with the fact that $h \le \rho/2$ --- which implies that $\xi \le 2h/\rho$ --- yields 
\[\d_{\cS}(x,x') - \d_\cM(x,x') \le \frac{2h}{\rho}\, \d_\cM(x,x') \quad \text{and} \quad \d_{\cM}(x,x') - \d_{\cS}(x,x') \le \frac{2h}{\rho}\, \d_{\cS}(x,x'),\]
from which the result follows immediately.
\end{proof}

\subsection{Simplexes}
\label{sec:simplexes}
A finite subset $\sigma$ of a Euclidean space is said to be a $k$-simplex if $\sigma$ is the convex hull of $k+1$ affinely independent points.
The thickness $\tau(\sigma)$ of a $k$-simplex $\sigma$ is defined as the ratio of its smallest altitude to its diameter. (A slightly different definition is given in~\cite{Boissonnat2018}, but the two notions are proportional to each other.) 

The thickness of a simplex is a measure of its regularity in that a lower bound on the thickness implies a lower bound on the angles of the simplex, and also on the ratio of the lengths of its shortest and longest edges. In particular, a regular $k$-simplex has the largest possible thickness among all $k$-simplexes, equal to $\tau_k := \sqrt{(k+1)/2k}$.

The thickness of a simplex $\sigma$ can also be measured based on its side length ratio $\pi(\sigma)$, defined as the length of its shortest edge divided by the length of its longest edge. 
Indeed, the following (straightforward) result holds.

\begin{lemma} \label{lem:thickness}
There is an increasing homeomorphism $\phi_k$ of $[0,1]$ such that $\tau(\sigma) \ge \tau_k\, \phi_k(\pi(\sigma))$ for any $k$-simplex $\sigma$.
\end{lemma}

%

\subsection{Affine subspaces}
\label{sec:affine}

Affine subspaces will play an important role in the form of tangent spaces. We will need the following bounds on the angle between affine subspaces. For two such subspaces, $T$ and $T'$, we denote by $\angle (T,T')$ their angle, or more precisely, their maximum principal (aka canonical) angle~\cite[Sec~I.5.2]{stewart1990matrix}.

The first result is referred to as Whitney's angle bound in~\cite{Boissonnat2018}.

\begin{lemma}[Lem~15c in~\cite{whitney1957geometric} or Lem~5.14 in~\cite{Boissonnat2018}] 
\label{lem:whitney} 
Let $T$ be an affine subspace and let $\sigma$ be a $k$-simplex whose edges are all of length at least $\eta$ and whose vertices are all within distance $\delta$ of $T$. Then 
\[\sin\angle({\rm aff}(\sigma), T) \le \frac2{(k-1)!} \frac{\delta}{\tau(\sigma) \eta},\] 
where ${\rm aff}(\sigma)$ is the affine subspace generated by the vertices of $\sigma$.
\end{lemma}

In the next result, we compare the distances of a point to two intersecting affine subspaces based on the angle between these subspaces.
\begin{lemma} 
\label{lem:dist_angle} 
For two intersecting affine subspaces $T$ and $T'$, and any point $x$, 
\begin{equation} 
    |\dist(x, T) - \dist(x, T')| 
    \le \angle (T, T') \dist(x, T \cap T'). 
\end{equation}
\end{lemma}

\begin{proof} 
Let $t \in T$, $t' \in T'$, and $y \in T \cap T'$ be closest to $x$ in their respective set. Define the angles 
\[\theta = \angle((xy), (ty)) = \angle((xy), T), \qquad 
\theta' = \angle((xy), (t'y)) = \angle((xy), T').\] 
Then $\dist(x, T) = \sin(\theta) \|x-y\|$ and $\dist(x, T') = \sin(\theta') \|x-y\|$, so that 
\begin{align*} 
    |\dist(x, T) - \dist(x, T')| 
     & \le |\sin\theta - \sin\theta'|\, \|x-y\|.
\end{align*}
We conclude with
\begin{align*} 
|\sin\theta - \sin\theta'|
\le |\theta - \theta'|
\le \angle(T,T'),
\end{align*}
by the triangle inequality for angles between subspaces, combined with 
\[\|x-y\| \le \dist(x, T \cap T'),\] 
due to the simple fact that $y \in T \cap T'$.
\end{proof}

Let $P_T$ denote the orthogonal projection onto the affine subspace $T$ (which is also the metric projection onto $T$). Also, for matrix $A$, let $\|A\|$ denote the operator norm of $A$. 
The following is well-known~\cite[Sec~I.5.2]{stewart1990matrix}.
\begin{lemma} \label{lem:proj_angle}
For two linear subspaces $T$ and $T'$ of same dimension, $\|P_{T} - P_{T'}\| = \sin\angle(T, T')$. Moreover, $\min\{\|U-{\rm I}\| : U \text{ orthogonal, } UT = T'\} = 2 \sin (\frac12 \angle(T, T'))$.
\end{lemma}

\subsection{Tangent spaces}
For a submanifold $\cM$, we let $T_\cM(x)$ denote the tangent space of $\cM$ at $x \in \cM$.
A lot is known about the tangent spaces of a submanifold with positive reach and their orthogonal projections.

The first result is on the distance of a point on the surface to a tangent space at some other point on the surface, and conversely, on the distance of a point on a tangent space to the surface.

\begin{lemma}[Th~4.18 in~\cite{Federer1959}; Lem~7.8(2) in~\cite{Boissonnat2018}; Lem~2 in~\cite{arias2017spectral}] 
\label{lem:tangent_distance} 
Let $\cM$ be a submanifold with reach at least $r > 0$. 
For $x, x' \in \cM$, $\dist(x', T_\cM(x)) \le \frac1{2r} \|x-x'\|^2$. 
Moreover, if $t \in T_\cM(x)$ is such that $\|t-x\| \le r/3$, then $\dist(t, \cM) \le \frac1r \|t-x\|^2$.
\end{lemma}

The second result is on the distortion of the projection onto a tangent space restricted to a neighborhood of the surface around the point of contact.

\begin{lemma}[Lem~7.14(1) in~\cite{Boissonnat2018} and Lem~5 in~\cite{arias2017spectral}] 
\label{lem:tangent_distortion} 
Let $\cM$ be a submanifold with reach at least $r > 0$. 
For any $h < r/2$ and any $x \in \cM$, the restriction to $B(x, h) \cap \cM$ of the orthogonal projection onto the tangent space at $x$ is a $4 (h/r)^2$-distortion map. 
Moreover the projection of $B(x, h) \cap \cM$ contains $B(x, h - C h^3) \cap T_\cM(x)$ for a constant $C$ that depends only on $\cM$.
\end{lemma}

%

\section{Minimax metric estimation}
\label{sec:distortion}

The basic idea leading to our new bound~\eqref{distance_error} is to reconstruct the surface, at least approximately, and then compute the shortest paths between the sample points on the reconstructed surface. In our case, it turns out that the reconstructed surface interpolates the sample points, but this is not necessary in principle.

\subsection{Metric estimation by surface reconstruction}
\label{sec:surface_reconstruction}
In this subsection, we are in a setting where we have a set of points $\bX = \{x_1, \dots, x_n\} \subset \bbR^d$ assumed, as in \eqref{eps}, to be an $\eps$-covering of a set $\cM \subset \bbR^d$ satisfying the following properties:
\begin{assumption} 
\label{asp:M} 
$\cM$ is a compact and connected $k$-dimensional $C^2$ submanifold without boundary.
\end{assumption}
See \remref{boundary} for extensions.

Our main goal is to define a surface $\hat\cM$ interpolating the same points and with similar characteristics. This surface will approximate $\cM$ well enough that the distances on $\hat\cM$ will be good approximations to the distances on $\cM$. The approach for defining $\hat\cM$ is not constructive, but rather relies on the axiom of choice. We present an actual construction in \secref{meshes} based on recent developments in computational geometry. Our definition here is much more elementary and is enough to establish the achievability of~\eqref{distance_error}, at least from an information-theoretic perspective.

Let $\bbM = \bbM(k, \bX, \eps)$ denote the class of submanifolds satisfying \aspref{M} for which $\bX$ is an $\eps$-covering. We know that $\bbM$ is non-empty since $\cM \in \bbM$. 
Let $\rho_{\rm max}$ denote the supremum reach among surfaces in $\bbM$. Select any surface\footnote{If there is a surface in $\bbM$ with reach $\rho_{\rm max}$, it is natural to choose such a surface. We believe this is possible, but we are not sure. In any case, what matters is that the regularity of $\hat\cM$ is controlled as a function of $\cM$.}  
$\hat\cM$ in $\bbM$ with reach $\rho(\hat\cM) \ge \rho_{\rm max}/2$, so that $\rho(\hat\cM) \ge \rho(\cM)/2$.
The surface $\hat\cM$ offers a good approximation to $\cM$, as the following result establishes.

\begin{proposition} 
\label{prp:approximation} 
There is $C$ which only depends on $\cM$ such that $\dist(\cM, \hat\cM) \le C \eps^2$.
\end{proposition}

With the interpolating surface $\hat\cM$ defined, we estimate the metric on $\cM$ by the metric on $\hat\cM$. Therefore, define the estimator
\begin{equation} \label{d_hat}
\hat d_{ij} := \d_{\hat\cM}(x_i, x_j), \quad \forall i, j \in [n].
\end{equation}

\begin{theorem} 
\label{thm:approximation} 
There is a constant $C$ that only depends on $\cM$ such that, whenever $\eps \le 1/C$, 
\begin{equation} 
    \label{approximation} 
    |\d_\cM(x,x') - \d_{\hat\cM}(x,x')| 
    \le C \eps^2 \min\{\d_\cM(x,x'), \d_{\hat\cM}(x,x')\}, \quad \forall x, x' \in \cM \cap \hat\cM. 
\end{equation} 
This implies that the estimator defined in \eqref{d_hat} satisfies \eqref{distance_error}.
\end{theorem}

\begin{proof} 
Let $C_0$ be the constant of \prpref{approximation}. 
Suppose that $\eps$ is small enough that $C_0 \eps^2 \le \rho(\cM)/4$. Because $\rho(\hat\cM) \ge \rho(\cM)/2$ by construction, we have $\dist(\cM, \hat\cM) \le \frac12 \min\{\rho(\cM), \rho(\hat\cM)\}$, which allows us to apply  \lemref{both_projections} and conclude.
\end{proof}

We now turn to the proof of \prpref{approximation}. In what follows, $C, C_1, C_2, \dots$ are generic constants that only depend on $\cM$ and may change with each appearance.

\begin{lemma} 
\label{lem:simplex} 
In the present situation, whenever $\eps \le 1/C$, for every $x \in \cM$, there are sample points $x_{i_1}, \dots, x_{i_k}$ such that the $k$-simplex defined by $\{x, x_{i_1}, \dots, x_{i_k}\}$ has minimum side length $\ge \eps$ and thickness $\ge 1/C$.
\end{lemma}

\begin{proof} 
Let $T$ be shorthand for $T_\cM(x)$. For $A > 0$ to be chosen large enough later, pick $t_1, \dots, t_k \in T$ such that the convex hull of $\{x, t_1, \dots, t_k\}$ is a regular $k$-simplex of side length $A \eps$. By \lemref{tangent_distortion}, the resulting map is one-to-one on $B(x, h) \cap \cM$ whenever $h < \rho(\cM)/2$. We restrict $P_T$ to that set, and for each $q$, define $u_q = P_T^{-1}(t_q)$ so that $u_q \in B(x, h) \cap \cM$. 
In particular, since $u_q \in \cM$, we have $\|t_q - u_q\| \le \dist(t_q, \cM)$. 
Noting that $\|t_q - x\| = A\eps$, by \lemref{tangent_distance} there is $C_1>0$ such that, if $A\eps \le 1/C_1$, then $\dist(t_q, \cM) \le C_1 (A \eps)^2$.

For each $q$, let $x_{i_q}$ be a sample point satisfying $\|u_q - x_{i_q}\| \le \eps$, which exists by virtue of the fact that $u_q \in \cM$ by construction and the sample points form an $\eps$-covering of $\cM$ by assumption. Let $\sigma$ denote the simplex defined by $\{x, x_{i_1}, \dots, x_{i_k}\}$, meaning the convex hull of that point set. Then, by the triangle inequality, $\sigma$ has side lengths satisfying 
\begin{align*} 
    \|x_{i_q} - x_{i_p}\| 
     & \le \|x_{i_q} - u_q\| + \|u_q - t_q\| + \|t_q - t_p\| + \|t_p - u_p\| + \|u_p - x_{i_p}\| \\ 
     & \le \eps + C_1 (A \eps)^2 + A \eps + C_1 (A \eps)^2 + \eps                                    \\ 
     & = A \eps\, (1 + 2/A + 2 C_1 A \eps), 
\end{align*} 
and 
\begin{align*} 
    \|x_{i_q} - x_{i_p}\| 
     & \ge -\|x_{i_q} - u_q\| - \|u_q - t_q\| + \|t_q - t_p\| - \|t_p - u_p\| - \|u_p - x_{i_p}\| \\ 
     & \ge -\eps - C_1 (A \eps)^2 + A \eps - C_1 (A \eps)^2 - \eps                                    \\ 
     & = A \eps\, (1 - 2/A - 2 C_1 A \eps), 
\end{align*} 
and the same upper and lower bounds apply when $x$ replaces $x_{i_q}$ above. 
Therefore, $\sigma$ has minimum side length $\ge A \eps\, (1 - 2/A - 2 C_1 A \eps)$ and side length ratio 
\[\pi(\sigma) \ge 
\frac{1 - 2/A - 2 C_1 A \eps}{1 + 2/A + 2 C_1 A \eps}.\]
We first require that $A \ge 4$ and $A \eps \le 1/8 C_1$, so that $\sigma$ has minimum side length $\ge \eps$.
Recall the definition of $\phi_k$ in \lemref{thickness}. Since $\tau_k > 1/\sqrt{2}$, there is $C_2$ such that $\phi_k(1-1/C_2) = 1/(\sqrt{2} \tau_k)$. Choose $A \ge 4$ large enough that $(1-2A)/(1+2A) \ge 1-1/2C_2$, and then $\eps$ small enough that $A \eps \le 1/8C_1$ (as required above) and $(1 - 2/A - 2 C_1 A \eps)/(1 + 2/A + 2 C_1 A \eps) \ge 1-C_2$. In that case, the simplex $\sigma$ constructed above is such that $\pi(\sigma) \ge 1-1/C_2$, which via \lemref{thickness} implies that $\tau(\sigma) \ge \tau_k \phi_k(1-1/C_2) = 1/\sqrt{2}$. 
\end{proof}

\begin{lemma} 
\label{lem:tangent_angle_intersection} 
In the present situation, for every $x \in \cM \cap \hat\cM$, $\angle(T_\cM(x), T_{\hat\cM}(x)) \le C \eps$.
\end{lemma}
 
\begin{proof} 
Take any point $x \in \cM \cap \hat\cM$ and consider the point set $x_{i_1}, \dots, x_{i_k}$ defined in \lemref{simplex}. 
By \lemref{tangent_distance}, 
\[\dist(x_{i_q}, T_\cM(x)) \le C_1 \|x_{i_q} - x\|^2 \le C_2 \eps^2.\] 
Using \lemref{whitney}, we have that  
\[\sin\angle({\rm aff}(\sigma), T_\cM(x)) \le A \frac{C_2\eps^2}{(1/C_0) \eps} =: C_3 \eps,\]  
where $A$ is a universal constant and $C_0$ is the constant of \lemref{simplex}. 
Similarly,  
\[\sin\angle({\rm aff}(\sigma), T_{\hat\cM}(x)) \le C_4 \eps.\]  
(In principle $C_4$ would depend on $\hat\cM$, but a more careful tracking of the constants reveal that they really only depend on a lower bound on the reach of the underlying surface, and $\rho(\hat\cM) \ge \rho(\cM)/2$ by construction.) 
We then conclude by the triangle inequality that 
\begin{align*} 
    \angle(T_\cM(x), T_{\hat\cM}(x)) 
     & \le \angle(T_\cM(x), {\rm aff}(\sigma)) + \angle({\rm aff}(\sigma), T_{\hat\cM}(x)) \\ 
     & \le \tfrac\pi2 C_3 \eps + \tfrac\pi2 C_4 \eps =: C_5 \eps,
\end{align*}
using the fact that $\sin a \ge \frac2\pi a$ for all $a \in [0,\frac\pi2]$.
\end{proof}

\begin{proof}[Proof of \prpref{approximation}] Take $x \in \cM$. We want to show that $\dist(x, \hat\cM) \le C \eps^2$.
Let $x_i$ be such that $\|x-x_i\| \le \eps$.
Then by \lemref{tangent_distance}, we have $\dist(x, T_\cM(x_i)) \le C_1 \eps^2$. Also, by \lemref{dist_angle},
\begin{align*} 
\dist(x, T_{\hat\cM}(x_i)) 
 & \le \dist(x, T_\cM(x_i)) + \angle(T_\cM(x_i), T_{\hat\cM}(x_i)) \dist(x, T_\cM(x_i) \cap T_{\hat\cM}(x_i)) \\ 
 & \le C_1 \eps^2 + (C_2 \eps) \|x - x_i\|
 \le C_3 \eps^2,
\end{align*}
using \lemref{tangent_angle_intersection} with $C_2$ denoting the constant there.
Let $t$ be the orthogonal projection of $x$ onto $T_{\hat\cM}(x_i)$ so that $\|x-t\| \le C_3 \eps^2$.
We have
\begin{align*} 
\|x_i - t\| 
 & \le \|x_i - x\| + \|x - t\|     \\ 
 & \le \eps + \dist(x, T_{\hat\cM}(x_i)) \\ 
 & \le \eps + C_3 \eps^2 
 \le C_4 \eps.
\end{align*}
By \lemref{tangent_distance}, if $C_4 \eps \le 1/C_5$, then $\dist(t, \hat\cM) \le C_5 (C_4 \eps)^2 =: C_6 \eps^2$.
($C_5$ depends monotonically on $\rho(\hat\cM)$, but $\rho(\hat\cM) \ge \rho(\cM)/2$ by construction.) 
Finally, by the triangle inequality,
\[\dist(x, \hat\cM) \le \|x-t\| + \dist(t, \hat\cM) \le C_3 \eps^2 + C_6 \eps^2 =: C_7 \eps^2.\]
Similarly, we can show that for any $s \in \hat\cM$, $\dist(s, \cM) \le C_8 \eps^2$, and combined, this allows us to conclude that $\dist(\cM, \hat\cM) \le \max\{C_7, C_8\} \eps^2$.
\end{proof}

\begin{remark} 
\label{rem:boundary} 
Although we have followed the tradition of working with submanifolds without boundary, an extension to submanifolds with boundary is straightforward albeit a little more tedious. 
Indeed, all the steps in the proof of \thmref{approximation} apply to surfaces with boundary, except possibly for \lemref{simplex}. 
For this lemma to apply, it is enough that $\partial\cM$ is itself smooth or that it does not have arbitrarily `sharp' singularities. Technically, it is enough that, for some constant $\alpha = \alpha(\cM) > 0$, for each $x\in\cM$ and $h \le \alpha$, the orthogonal projection of $B(x,h) \cap \cM$ onto $T_\cM(x)$ contains a cone of the form 
\begin{equation}
\big\{t \in T_\cM(x) \cap B(x,\alpha h) : \angle(t-x,u) \le \alpha\big\},
\end{equation}
for some normed vector $u$. This condition applies, for example, to the situation where $\partial\cM$ is itself a $C^2$ submanifold, or more generally, when it is locally the graph of a Lipschitz function. 
\end{remark}

\subsection{Information bound}
\label{sec:information}
So far, we have worked in a setting where all we know about the underlying surface $\cM$ satisfies \aspref{M}. 
Based on this, we have defined an estimator that satisfies the error bound~\eqref{distance_error}. The questions arises: Is this best possible? We approach this question from a minimax perspective, and it turns out it is indeed best possible.

Our approach is standard: the idea is to construct a situation where two distinct surfaces satisfying \aspref{M} with sufficiently different metrics and that interpolate the same set of points. For other examples in the geometrical statistics literature, see~\cite{genovese2012manifold, kim2015tight, aamari2018stability, divol2020minimax, aamari2022optimal}.
We work with surfaces that have a boundary, knowing that we can extend them into surfaces without boundary without modifying the construction in any otherwise meaningful way.

In the next subsections, we prove the following information lower bound. (In fact, we prove a somewhat stronger result.)
\begin{theorem}
For any estimator $\hat d$, the following is true. 
For any $\eps > 0$, there is a surface $\cM$ satisfying \aspref{M} and a set of points $x_1, \dots, x_n$ belonging to $\cM$ dense enough that \eqref{eps} holds, such that the proportion of pairs $i \ne j$ for which
\begin{equation}
|\hat d_{ij} - \d_\cM(x_i, x_j)| \ge C^{-1} \eps^2 \d_\cM(x_i, x_j)
\end{equation}
approaches~1 as $C$ increases without bound. 
\end{theorem}

\subsubsection{Case $k=1$}
\label{sec:k=1}
As a warm-up, we consider the case where the underlying submanifold is a curve, meaning of dimension $k = 1$. It is enough to consider the plane ($d = 2$). There, let $\cM_1$ be defined as the line segment $[0,1] \times \{0\}$. Starting at the origin and moving right, place sample points $\eps$ apart, and assume for convenience that $\eps = 1/(n-1)$.
The sample points are therefore $x_i = ((i-1) \eps, 0)$ for $i = 1, \dots, n$.
We define $\cM_2$ by bending and stretching $\cM_1$.
To create an `arc' between two sample points, we use a $C^2$ function $w$ supported on $[-1,1]$ and such that $w(0) = 1$. Define $\cM_2$ by changing in $\cM_1$ the line segment joining $x_i$ and $x_{i+1}$ into the curve given by the graph of the function $t \mapsto A (\eps/2)^2 w((t-(i-1/2)\eps)/(\eps/2))$ on the interval $[(i-1)\eps, i\eps]$, doing so for each $i = 1, \dots, n$. Let $f_\eps: [0,1] \to \bbR$ denote the resulting function and $\gamma_\eps(t) = (t, f_\eps(t))$, and
set $\cM_2 = \gamma_\eps([0,1])$.
By construction, $\cM_2$ is a $C^2$ simple curve with curvature bounded from above by a universal constant multiple of $A$.
The parameter $A>0$ is fixed and only there to indicate that any upper bound on the curvature can be fulfilled by choosing $A$ sufficiently small. The dependence on $A$ is otherwise left implicit, as it is of secondary importance.
See \figref{minimax_curves} for an illustration.

\begin{figure}[tbp!] 
\centering 
\includegraphics[scale = 0.5]{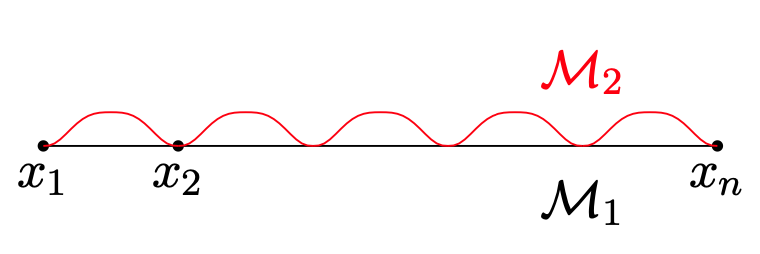} 
\caption{The case of dimension $k=1$ in ambient dimension $d=2$ (the latter without loss of generality). What is pictured is in fact a piece of each curve. The reader is invited to imagine that these curves are completed by the same curve, say $\cM_0$, so that $\cM_0 \cup \cM_1$ and $\cM_0 \cup \cM_2$ are smooth simple closed curves. Also, additional points would be placed on $\cM_0$ so that, together with the points depicted in the plot, they would form an $\eps$-covering.} 
\label{fig:minimax_curves}
\end{figure}

Take $1 \le i < j \le n$. While the distance between $x_i$ and $x_j$ on $\cM_1$ is obviously $(j-i)\eps$, their distance on $\cM_2$ is equal to the length of the piece of $\cM_2$ starting at $x_i$ and ending at $x_j$, which is equal to $(j-i) \eta$, where $\eta$ is the length of the piece of $\cM_2$ between $x_1$ and $x_2$.
Elementary calculations show that $\eta \ge \eps + C_1 \eps^3$ for some $C_1 > 0$ which depends only on $A$. 
To be sure, assume that $\eps$ is small enough that $A (\eps/2) \sup_t |w'(t)| \le 1$, and compute
\begin{align*} 
\eta 
 & = \int_0^\eps \Big\{1 + \big[A (\eps/2)^2 (2/\eps) w'(2t/\eps-1)\big]^2\Big\}^{1/2} \d t                                          \\ 
 & = \int_{-1}^1 \Big\{1 + \big[A (\eps/2) w'(t)\big]^2 \Big\}^{1/2} (\eps/2) \d t                                                   \\ 
 & \ge \int_{-1}^1 \Big\{1 + \tfrac14 \big[A (\eps/2) w'(t)\big]^2 \Big\} (\eps/2) \d t                                                   
 = \eps + C_1 \eps^3, \qquad C_1 := \tfrac1{32} A^2 \int_{-1}^1 w'(t)^2 \d t.
\end{align*}
Similar calculations show that $\eta \le \eps + C_2 \eps^3$, for another constant $C_2$ depending only on $A$. We have thus bounded $\eta$ from below and above as follows 
\begin{equation}
\label{eta_bounds}
\eps + C_1 \eps^3 \le \eta \le \eps + C_2 \eps^3.
\end{equation}
Using the lower bound in \eqref{eta_bounds}, we get that the distance on $\cM_2$ between $x_i$ and $x_j$ is $\ge (j-i)(\eps + C_1 \eps^3)$. In particular, because the distance on $\cM_1$ between $x_i$ and $x_j$ is $= (j-i)\eps$, when $\eps$ is sufficiently small, we have
\begin{equation} 
\label{info1} 
\d_{\cM_2}(x_i, x_j) - \d_{\cM_1}(x_i, x_j) 
\ge (j-i) C_1 \eps^3 
\ge C_1 \eps^2 \d_{\cM_1}(x_i, x_j).
\end{equation}
Using the upper bound in \eqref{eta_bounds}, we get that
\begin{equation} 
\label{info1_upper} 
\d_{\cM_2}(x_i, x_j) - \d_{\cM_1}(x_i, x_j) 
\le C_2 \eps^2 \d_{\cM_1}(x_i, x_j) 
\le \d_{\cM_1}(x_i, x_j),
\end{equation}
with the last inequality holding as soon as $\eps$ is small enough that $C_2 \eps^2 \le 1$.

Based on what we know of the true $\cM$, it could be $\cM_1$ as easily as $\cM_2$, and therefore, for any estimate $\hat d_{ij}$,
\begin{align*}
\max_{\cM \in \{\cM_1, \cM_2\}}    \big|\hat d_{ij} - \d_\cM(x_i, x_j)\big| 
 & \ge \tfrac12\, \big| \d_{\cM_2}(x_i, x_j) - \d_{\cM_1}(x_i, x_j) \big|  \\ 
 & \ge \tfrac14 C_1 \eps^2\, \d_{\cM_1}(x_i, x_j)                               \\ 
 & \ge \tfrac18 C_1 \eps^2\, \max\{\d_{\cM_2}(x_i, x_j), \d_{\cM_1}(x_i, x_j)\} \\ 
 & = \tfrac18 C_1 \eps^2\, \max_{\cM \in \{\cM_1, \cM_2\}}\d_{\cM}(x_i, x_j).
\end{align*}


\subsubsection{Case $k \ge 2$}
In general, it is enough to consider the setting where $d = k+1$. 
We also consider a regular grid where $x_i = (x_{i,1}, \dots, x_{i,k}, 0)$ with $i = (i_1, \dots, i_k)$ and $x_{i,q} = (i_q-1)\eps$ for $i_q = 1, \dots, m$ and $q = 1, \dots, k$. We are indeed assuming that $n$ is of the form $n = m^k$ for some positive integer $m$. This is for convenience, although again, it brings the focus to what should in principle be a regular case (since the sample points are well spread out). We again assume for expediency that $\eps = 1/(m-1)$.
We take $\cM_1 = [0,1]^k \times \{0\}$, and
\[\cM_2 = \big\{(t_1, \dots, t_{k-1}, \gamma_\eps(t_k)) : t_1,\dots,t_k \in [0,1]\big\},\]
where $\gamma_\eps$ is the same parameterized curve that was constructed in \secref{k=1}.
Clearly, the sample points belong to both surfaces.
See \figref{minimax_surfaces} for an illustration.

\begin{figure}[tbp!] 
\centering 
\includegraphics[scale = 0.5]{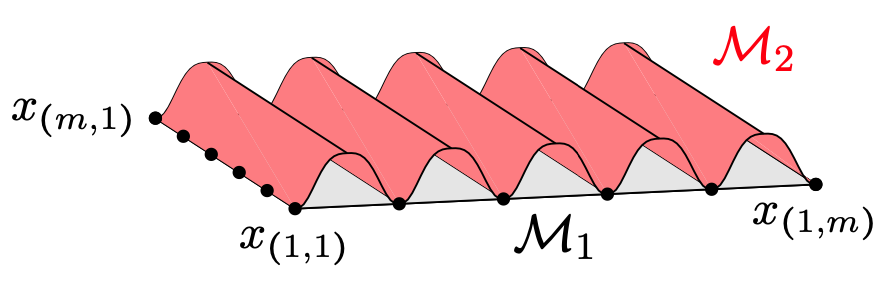} 
\caption{Analogous to \figref{minimax_curves}, but in the case of dimension $k=2$ in ambient dimension $d=3$.} 
\label{fig:minimax_surfaces}
\end{figure}

On the one hand, $\cM_1$ is convex, and thus the intrinsic metric on $\cM_1$ coincides with the Euclidean metric. In particular,
\begin{align*} 
\dist_{\cM_1}(x_i, x_j)^2 
 & = \|x_i - x_j\|^2                  \\ 
 & = \sum_{q=1}^k (i_q-j_q)^2 \eps^2,
\end{align*}
for all $i,j \in [m]^k$.
On the other hand, recalling the definition of $\eta$ given in \secref{k=1}, by straightening along the $(k+1)$th canonical direction, we see that $\cM_2$ is isometric to $[0,1]^{k-1} \times [0, (m-1) \eta]$ via the isometry 
\[(t_1, \dots, t_{k-1}, \gamma_\eps(t_k))) \mapsto (t_1, \dots, t_{k-1}, \lambda_\eps(t_k)),\]
where 
\begin{equation} \label{lambda}
\lambda_\eps(t) := \Lambda\big(\gamma_\eps([0,t])\big).
\end{equation}
This isometry is based on an arc length parameterization of $\gamma_\eps$. See \lemref{arclength}.
In particular, with this isometry 
\[x_i \mapsto u_i := (x_{i,1}, \dots, x_{i,k-1}, \lambda_\eps (x_{i,k})) = ((i_1-1)\eps, \dots, (i_{k-1}-1)\eps, (i_k-1)\eta).\]
As a consequence,
\begin{align*} 
\dist_{\cM_2}(x_i, x_j)^2 
 & = \|u_i - u_j\|^2                                           \\ 
 & = \sum_{q=1}^{k-1} (i_q-j_q)^2 \eps^2 + (i_k-j_k)^2 \eta^2,
\end{align*}
for all $i,j \in [m]^k$.
Hence,
\begin{align*} 
\dist_{\cM_2}(x_i, x_j)^2 - \dist_{\cM_1}(x_i, x_j)^2 
= (i_k-j_k)^2 (\eta^2 - \eps^2),
\end{align*}
and again, this is so for all $i,j \in [m]^k$.
Using the upper bound in \eqref{eta_bounds}, if $\eps$ is sufficiently small that $C_2 \eps^2 \le 1$, we get
\begin{align*} 
\dist_{\cM_2}(x_i, x_j)^2 - \dist_{\cM_1}(x_i, x_j)^2 
 & = (i_k-j_k)^2 (\eta-\eps)(\eta+\eps)                                          \\ 
& \le (i_k-j_k)^2 C_2 \eps^3 (2\eps+C_2\eps^3) \\ 
& \le (i_k-j_k)^2 3 C_2 \eps^4 \\ 
& = 3 C_2 \eps^2 \beta_{ij}^2 \dist_{\cM_1}(x_i, x_j)^2 \\
&\le 3 \dist_{\cM_1}(x_i, x_j)^2,
\end{align*}
where $\beta_{ij} := \cos(\theta_{ij})$ and $\theta_{ij}$ is the angle that the line passing through $x_i$ and $x_j$ makes with the $k$th axis.
In the process, we found that
\[\dist_{\cM_2}(x_i, x_j) \le 2 \dist_{\cM_1}(x_i, x_j).\]
From this, we get
\begin{align*} 
\dist_{\cM_2}(x_i, x_j) - \dist_{\cM_1}(x_i, x_j) 
&\le \frac{3 C_2 \eps^2 \beta_{ij}^2 \dist_{\cM_1}(x_i, x_j)^2}{\dist_{\cM_2}(x_i, x_j) + \dist_{\cM_1}(x_i, x_j)} \\
&\le C_2 \eps^2 \beta_{ij}^2 \dist_{\cM_1}(x_i, x_j).
\end{align*}
Using the lower bound in \eqref{eta_bounds}, we get
\begin{align*} 
\dist_{\cM_2}(x_i, x_j)^2 - \dist_{\cM_1}(x_i, x_j)^2 
 & = (i_k-j_k)^2 (\eta-\eps)(\eta+\eps)                                          \\ 
 & \ge (i_k-j_k)^2 C_1 \eps^4 \\ 
 & = C_1 \eps^2 \beta_{ij}^2 \dist_{\cM_1}(x_i, x_j)^2,
\end{align*}
and, as before, this implies that
\begin{align*} 
\dist_{\cM_2}(x_i, x_j) - \dist_{\cM_1}(x_i, x_j) 
\ge \tfrac13 C_1 \eps^2 \beta_{ij}^2 \dist_{\cM_1}(x_i, x_j).
\end{align*}

Based on what we know of the true $\cM$, it could be $\cM_1$ as easily as $\cM_2$, and therefore, for any estimate $\hat d_{ij}$,
\begin{align*} 
\max_{\cM \in \{\cM_1, \cM_2\}}    \big|\hat d_{ij} - \d_\cM(x_i, x_j)\big| 
 & \ge \tfrac12\, \big| \d_{\cM_2}(x_i, x_j) - \d_{\cM_1}(x_i, x_j) \big|                        \\ 
 & \ge \tfrac16 C_1 \beta_{ij}^2 \eps^2\, \d_{\cM_1}(x_i, x_j)                                  \\ 
 & \ge \tfrac1{12} C_1 \beta_{ij}^2 \eps^2\, \max\{\d_{\cM_2}(x_i, x_j), \d_{\cM_1}(x_i, x_j)\} \\ 
 & = \tfrac1{12} C_1 \beta_{ij}^2 \eps^2\, \max_{\cM \in \{\cM_1, \cM_2\}}\d_{\cM}(x_i, x_j).
\end{align*}
We conclude with the fact that as $t \searrow 0$, the proportion of pairs $(i,j)$ such that $\beta_{ij}^2$ exceeds $t$ tends to~1.


\section{Meshes}
\label{sec:meshes}

Polytopes form an important class of surfaces used in computational geometry, numerical partial differential equations, and more. Their approximation properties and their simplicity allow for the design of algorithms for rendering a surface on a computer under a chosen lighting condition, as done in 3D animation, or for simulating a particular equation arising in physics.
Among polytopes that are routinely used in practice, simplicial complexes are arguably the most common. They are particularly relevant in our context as they are often used to assess the shape defined by an otherwise unorganized set of points.
A finite collection $K$ of simplexes constitutes a simplicial complex if it is closed under intersection (i.e., the intersection of two simplexes in $K$ is either empty or itself a simplex of $K$) and if any face of a simplex in $K$ is also a simplex in $K$.

When used to approximate of a surface, a simplicial complex is often called a mesh. A number of mesh construction are available in the literature, including some that come with theoretical guarantees --- see \tabref{meshes} for some prominent examples. 
Because of the available theory, we work with the tangential Delaunay complex.

\begin{table}[htbp!] 
\centering
\caption{Overview of some mesh constructions with their theoretical guarantees.} 
\label{tab:meshes} 
\footnotesize\setlength{\tabcolsep}{0.0in} 
\begin{tabular}{p{0.4\textwidth} p{0.6\textwidth}} 
    \toprule 
    {\bf Method}                                                                          & {\bf Guarantees}                                                                    \\ \midrule 
    Ball-pivoting~\cite{Bernardini1999, digne2014analysis}                                & homeomorphism; watertight; bound in Hausdorff distance                          \\ \midrule 
    Crust~\cite{Amenta1998a, Amenta1998, amenta1999crust, amenta1999surface},\ 
    Cocone~\cite{amenta2000simple}                                                        & homeomorphism; bound in Hausdorff distance and in angle                            \\ \midrule 
    Tight Cocone~\cite{dey2003tight, dey2006provable}, Power Crust~\cite{amenta2001power} & homeomorphism; watertight; bound in Hausdorff distance and in angle             \\  \midrule 
    Noisy Power Crust~\cite{dey2006provable, mederos2005surface}                          & homeomorphism 
   \\ \midrule 
    Natural Neighbors~\cite{Hoppe1992, boissonnat2002smooth}                              & homeomorphism; convergence in the Hausdorff metric 
    \\ \midrule 
    Peel~\cite{dey2009isotopic}                                                           & isotopy; convergence in Hausdorff metric 
    \\ 
    \bottomrule 
\end{tabular}
\end{table}

\subsection{Nets}
\label{sec:nets}

The mesh construction we use --- introduced in \secref{TDC} below --- requires that the data points form an $\eps$-covering of $\cM$, meaning that \eqref{eps} holds, and that they form a $c \eps$-packing for some constant $c > 0$, meaning that 
\begin{equation} \label{packing}
\min_{i \ne j} \|x_i - x_j\| \ge c \eps.
\end{equation}

Suppose we are interested in estimating $\d_\cM(x_i, x_j)$ for a given pair of points indexed by $i, j \in [n]$.
If $\|x_i - x_j\| \le \eps$, let the estimate be $\hat d_{ij} := \|x_i - x_j\|$. 
\begin{lemma}[Lem~3 in~\cite{bernstein2000graph};  Lem~3.12 in \cite{arias2019unconstrained}]\label{lem:distance_comparison}
For $\cM$ satisfying \aspref{M}, there is a constant $C$ such that
\[0 \le \d_\cM(x, x') - \|x - x'\| \le C \|x-x'\|^3, \quad \forall x, x' \in \cM.\]
\end{lemma}
Applying this lemma, we see that the bound in \eqref{distance_error} applies for $i, j$, since
\begin{align}
0 
\le \d_\cM(x_i, x_j) - \hat d_{ij}
&\le C \|x_i-x_j\|^3 \\
&\le C \eps^2 \hat d_{ij} \\
&\le C \eps^2 \min\{\d_\cM(x_i, x_j), \hat d_{ij}\},
\end{align}
using the fact that $\hat d_{ij} = \|x_i-x_j\| \le \eps$.

If $\|x_i - x_j\| > \eps$, do as follows. Starting with $S_0 := \{x_i, x_j\}$, at stage $t$, add a data point to $S_t$ not within distance $\eps$ from a point in $S_t$ to form $S_{t+1}$ --- stop at $S_t$ if no such point exists. Let $S_\infty$ denote the resulting subset of sample points. By construction, any two points in $S_\infty$ are separated by a distance exceeding $\eps$. Also, any sample point not included in $S_\infty$ is within distance $\eps$ of a point in $S_\infty$, so that $S_\infty$ is an $2\eps$-cover of $\cM$ by the triangle inequality.
Following~\cite{Boissonnat2018}, $S_\infty$ is an $(2\eps, 1/2)$-net of $\cM$. (In general, a $(\eta, c)$-net of $\cM$ is a point set which is an $\eta$-covering of $\cM$ where any two points are at least $c\,\eta$ apart.)
We then construct a mesh based on $S_\infty$ and let $\hat d_{ij}$ denote the distance on $\hat\cM$ between $x_i$ and $x_j$.

Therefore, in without loss of generality, in the remaining of this section, we simply assume that $\bX$ itself is an $(\eps, 1/2)$-net of $\cM$. 


\begin{remark} \label{rem:nets}
Proceeding as a describe here would in principle require that we build a different mesh for each pair of points $x_i$ and $x_j$ such that $\|x_i-x_j\| > \eps$. This would appear to be wasteful and unnecessary in practice. We believe this is indeed the case. See \secref{discussion} for a longer discussion.
\end{remark}

The construction below also requires that the sample points be in general position and that a certain `transversality' condition\footnote{The condition is that no tangent space at any of the sample points contains a point that is equidistant to more than $k + 1$ points of $\bX$.} be satisfied. If these conditions are not already satisfied, they can be achieved by a simple infinitesimal random perturbation of the sample points, and so we assume they are satisfied in what follows.

\subsection{Tangential Delaunay complex}
\label{sec:TDC}
The tangential Delaunay complex is a mesh construction that dates back to~\cite{Freedman2002, Boissonnat2002}. Here we follow the exposition given in~\cite[Ch 8]{Boissonnat2018}.
In addition to the point set, $\{x_1, \dots, x_m\}$, the construction relies on knowledge of the tangent space at each sample point, meaning the knowledge of the tangent spaces at the sample points. We will see later in \secref{tangent_estimation} that these tangent spaces can be estimated to enough precision to circumvent this otherwise substantial requirement.

\begin{figure}[htbp!] 
\includegraphics[width=0.49\linewidth, trim=10 50 10 100, clip]{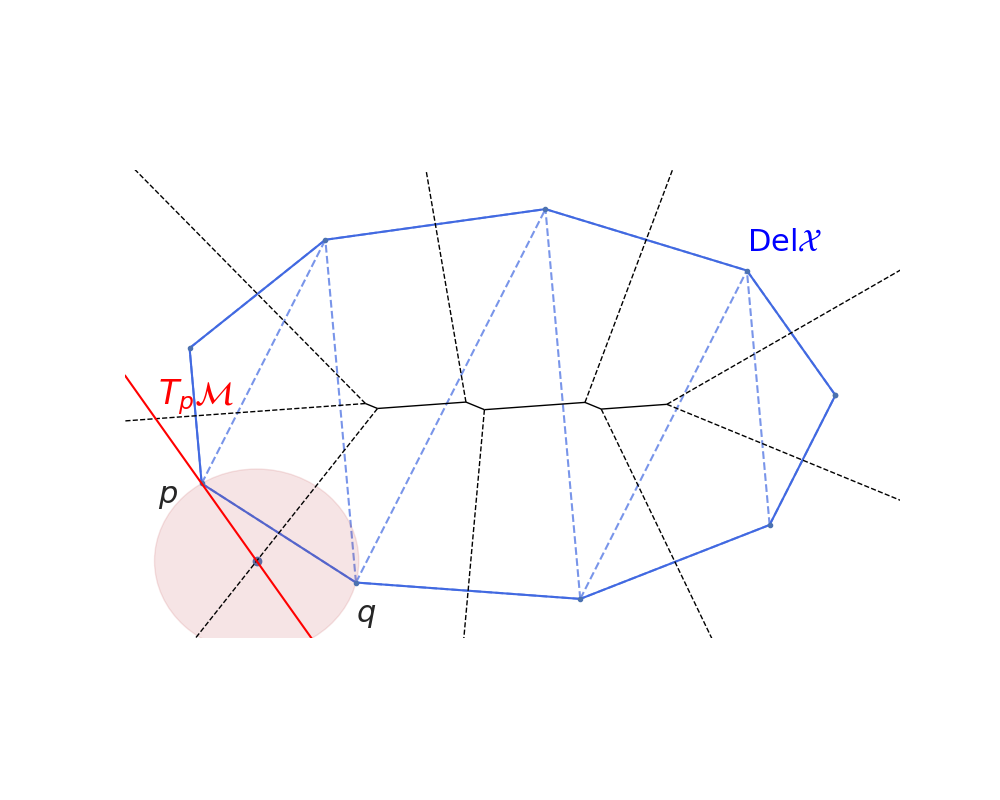} 
\includegraphics[width=0.49\linewidth, trim=10 50 10 100, clip]{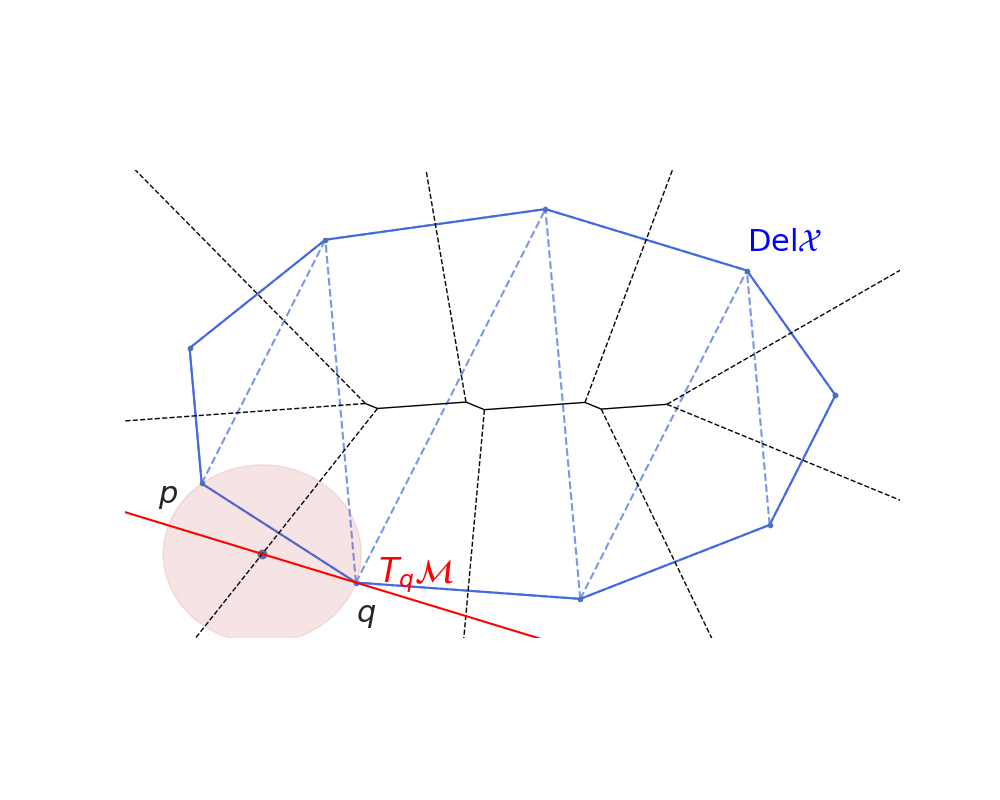} 
\caption{The Delaunay triangulation of a point cloud sampled from an ellipse is depicted in blue. 
    The dual Voronoi diagram is in black. The edge $\sigma=\overline{pq}$ is a consistent simplex of the Tangential 
    Delaunay complex because $\sigma\in\text{star}(p)\cap\text{star}(q)$. That is, 
    $\sigma$ can be circumscribed by empty balls centered on $T_p\cM$ and $T_q\cM$. (The 
    balls, although very similar in this example, are not the same in general.)}
\end{figure}

Let $\cU$ denote the Delaunay complex of $\bX$, which is the collection of all the simplexes with vertices in $\bY$ that admit a circumscribing ball empty of sample points in its interior.
For $i \in [n]$, let $\cU_i$ denote the Delaunay complex of $\bX$ restricted to the tangent space $T_i := T_\cM(x_i)$, which is defined as the subcomplex of $\cU$ formed by all the simplexes that admit a  circumscribing ball centered on $T_i$ empty of sample points  in its interior.
The closed star of $x_i$ in $\cU_i$, denoted $\cS_i$, is the subcomplex of $\cU_i$ that consists of the simplexes incident to $x_i$ together with all their faces.
With these definitions in place, the tangential Delaunay complex of $\bY$ is the simplicial complex made of the union of all these closed stars, i.e.,
\begin{equation} 
\cT := \big\{\sigma : \sigma \in \cS_i \text{ for some } i \in [n]\big\}.
\end{equation}
Because of the transversality condition mentioned above, $\cT$ does not contain faces of dimension greater than $k$.


When used to approximate a surface, the presence of thin simplexes or slivers (i.e., simplexes with small thickness as defined in \secref{simplexes}) in a mesh can make restrict the accuracy of the approximation to the underlying surface to 0th order and be completely inaccurate at the level of the tangent spaces. This is due to the fact that slivers can be make an arbitrarily large angle with the surface. In the extreme case, a sliver can even be perpendicular to the surface --- think of three points along a same great circle on a 2D sphere. 
The Schwarz lantern provides a famous example of this: it is an arbitrarily fine mesh of a cylinder which converges in Hausdorff metric (i.e., 0th order) while the simplexes never become tangent to the cylinder in the infinitesimally fine mesh limit.
In addition, thin simplexes can prevent the mesh from being a topological (in fact, piecewise linear) submanifold. 

A simplex of the tangential Delaunay complex $\cT$ is said to be inconsistent if it does not belong to the closed stars of all of its vertices.
In the presence of inconsistent simplexes, $\cT$ is not a topological submanifold. 
Without going into too much detail, to each inconsistent $k$-simplex of the tangential Delaunay complex, $\sigma \in \cT$, we can associate a $(k+1)$-simplex of the Delaunay complex, $\sigma^I \in \cU$, that is said to `trigger' the inconsistency; and, as it happens, that simplex $\sigma^I$ cannot be too thick~\cite[Cor 8.13]{Boissonnat2018}.
These inconsistencies are dealt with in~\cite{Boissonnat2018} by perturbing $\cU$  using a variant of the weighing method of~\cite{cheng2000silver} which consists in lifting the points to dimension $d+1$ by assigning weights to them and then reassigning random weights (from a carefully chosen distribution) to the vertices of a sliver; see Algorithm 5 in~\cite{Boissonnat2018}.
The overall method for building a tangential Delaunay complex with no inconsistencies is described in Algorithm 8 in~\cite{Boissonnat2018}, and is shown to have expected running time proportional to the sample size~\cite[Th~8.17]{Boissonnat2018}. 
We refer to this algorithm as TDC.
The overall algorithm is complex, but a relatively accessible although partial description is given in~\cite{aamari2018stability}.

\begin{theorem}[Th~7.16 and Th~8.18 in~\cite{Boissonnat2018}] 
\label{thm:TDC} 
There is a constant $C>0$ depending on $\cM$ such that, if TDC is provided with an $(\eps, 1/2)$-net of $\cM$ together with the tangent spaces at each point of the net, then with probability~1 it returns a piecewise linear submanifold of dimension $k$ without boundary that is a $C\eps^2$-distortion of $\cM$.
\end{theorem}

We provide a sketch of a roadmap through the book of Boissonnat {\it et al}~\cite{Boissonnat2018} that leads to the result.
Let $\cT_{\rm c}$ denote the output of TDC, that is, the tangential Delaunay complex built on the provided sample of points together with the accompanying tangents spaces --- corrected for inconsistencies.
According to~\cite[Th~8.18]{Boissonnat2018} the simplexes of $\cT_{\rm c}$ all have thickness at least $1/C_0$, and $\cT_{\rm c}$ and $\cM$ are within Hausdorff distance $C_0 \eps^2$. The proof of that result consists in large part in verifying that the conditions of~\cite[Th~7.16]{Boissonnat2018} are satisfied, which in particular includes showing that the simplexes of $\cT_{\rm c}$ all have diameter at most $C_0 \eps$ and also that the projection map $P_\cM : \cT_{\rm c} \to \cM$ (which is well-defined when $C_0 \eps^2 < \rho(\cM)$, which we assume is the case) is one-to-one. To complete the picture, the projection map is shown to be a $O(\eps^2)$-distortion map. By~\cite[Lem~7.13]{Boissonnat2018} (or our \lemref{pi_lipschitz}), this is true on each simplex of $\cT_{\rm c}$ by the above bounds on the thickness and diameter, and thus true on the entirety of $\cT_{\rm c}$ seen as a surface.

\subsection{Estimation of the tangential Delaunay complex}

We just saw that the tangential Delaunay complex, after correction for inconsistencies, provides a good enough approximation to the underlying surface for the metric approximation~\eqref{distance_error} to hold. In addition, from an algorithmic standpoint, the complex can be built in (randomized) polynomial time, and results in a piecewise linear surface for which algorithms for computing distances exist (see \secref{numerics}). All that said, the construction relies on knowing the tangent spaces at the sample points, which in principle is not part of the information we have access to.

As it turns out, this additional information is not needed: In the same setting, the tangent spaces can be estimated to enough accuracy that the construction of the tangential Delaunay complex based on these estimated tangent spaces, again after correction for inconsistencies, also provides a good enough approximation to the underlying surface. All we need is a lower bound on the reach of $\cM$ --- a more reasonable requirement.

The same strategy is considered by Aamari and Levrard~\cite{aamari2018stability}. 
We follow in their footsteps to obtain the desired bound on the distortion between the estimated tangential Delaunay complex and the underlying surface.

\subsubsection{Estimating the tangent spaces}
\label{sec:tangent_estimation}
The estimation of tangent spaces is by local principal component analysis, a natural approach used throughout the manifold estimation literature and manifold learning literature (e.g., in~\cite{arias2017spectral, kambhatla1997dimension, weingessel2000local, fukunaga1971algorithm}). While~\cite{aamari2018stability} works with a random sample, we show below that the same accuracy results if we work instead with an $(\eps, 1/2)$-net as we do here.

Let $\hat T_i$ denote the $k$-dimensional affine space that passes through $x_i$ and is parallel to the top $k$-dimensional eigenspace of the following matrix
\begin{equation} \label{Sigma}
\Sigma_i := \frac1{|N_i|} \sum_{j \in N_i} (x_j - x_i) (x_j - x_i)^\top, \qquad N_i := \big\{j \in [n]: \|x_j - x_i\| \le h\big\}.
\end{equation}
That eigenspace will be shown to be well-defined when $h$ is chosen proportional to $\eps$ and $\eps$ is small enough.
Unlike~\cite{aamari2018stability}, this matrix is not the covariance matrix of the sample points in $B(x_i, h)$ as it is centered at $x_i$ and not at the barycenter of those points. This is not essential, but helps streamline the proof of the following result. 

Below, we will use $\preceq$ to denote the Loewner order when comparing symmetric matrices of same size. 
For a symmetric matrix $M$, $\lambda_1(M) \ge \lambda_2(M) \ge \cdots$ denote its eigenvalues thus ordered. 

\begin{proposition}
\label{prp:tangent_estimation}
Choose $h = A \eps$ in \eqref{Sigma} for a constant $A$ depending only on $k$ specified below. There is $C > 0$ depending only on $\cM$ such that $\angle(\hat T_i, T_i) \le C \eps$ for all $i$.
\end{proposition}

\begin{proof}
In what follows, $A = 3/\eta_0$ where $\eta_0$ is implicitly defined in \lemref{Sigma_proj}. We only need to prove the statement for $\eps$ sufficiently small because an angle is bounded.

Let $P_i$ be short for $P_{T_i}$. 
Let $t_j$ denote the orthogonal projection of $x_j$ onto $T_i$. First, note that $t_i = x_i$, and for $j \in N_i$, 
\[\|t_j - x_j\| = \dist(x_j, T_i) \le C_1 \|x_j-x_i\|^2 \le C_1 h^2,\] 
by \lemref{tangent_distance}, so that  
\[\|t_j - t_{j'}\| \ge \|x_j - x_{j'}\| - \|t_j-x_j\| - \|t_{j'}-x_{j'}\| \ge \eps/2 - C_1 h^2 - C_1 h^2,\] 
by the triangle inequality. 

Next, we claim that $\{t_j : j \in N_i\}$ forms an $(\eps+C_2h^3)$-covering of $B(x_i, h) \cap T_i$, where $C_2$ is the constant of \lemref{tangent_distortion}. Indeed, take $t \in B(x_i, h) \cap T_i$ and let $t' \in B(x_i, h-C_2h^3) \cap T_i$ be such that $\|t'-t\| \le C_2 h^3$. By \lemref{tangent_distortion}, there is $x' \in B(x_i, h) \cap \cM$ such that $P_i(x') = t'$. Since $\{x_j : j \in N_i\}$ is an $\eps$-covering of $B(x_i, h-\eps) \cap \cM$, there must be $j \in N_i$ be such that $\|x_j - x'\| \le \eps$. Then $\|t_j-t'\| \le \eps$ because $P_i$ is 1-Lipschitz. By the triangle inequality, we get $\|t_j-t\| \le \eps + C_2 h^3$.
Note that, for any $j \in N_i$, $t_j \in B(x_i, h) \cap T_i$, since $\|t_j-t_i\| \le \|x_j-x_i\| \le h$, again relying on $P_i$ being 1-Lipschitz.  
Hence, recalling that $h = A\eps$, if $\eps$ is small enough that
\begin{align*}
\eps/2 - 2 C_1 (A\eps)^2 > \eps/3 \quad \text{and} \quad
\eps+C_2(A\eps)^3 \le 2\eps,
\end{align*}
we have that $\{t_j : j \in N_i\}$ forms a $(2\eps,1/6)$-net of $B(x_i, h) \cap T_i$.
   
Now, remembering that $t_i = x_i$, we have
\begin{align*} 
    \Sigma_i 
     & = \Sigma_i^{\rm proj} + R_1 + R_1^\top + R_2, \qquad \text{with} \quad
\Sigma_i^{\rm proj} := \frac1{|N_i|} \sum_{j \in N_i} (t_j-t_i)(t_j-t_i)^\top, \quad 
\end{align*}
and remainders
\begin{align*}
R_1 := \frac1{|N_i|} \sum_{j \in N_i} (x_j-t_j)(t_j-t_i)^\top, \qquad
R_2 := \frac1{|N_i|} \sum_{j \in N_i} (x_j-t_j)(x_j-t_j)^\top.    
\end{align*}
satisfying
\[\|R_1\| \le \max_{j \in N_i} \|x_j-t_j\| \|t_j-t_i\| \le (C_1 h^2) h = C_1 h^3,\]
\[\|R_2\| \le \max_{j \in N_i} \|x_j-t_j\|^2 \le \big(C_1 h^2\big)^2 = C_3 h^4.\]
Hence, 
\[\|\Sigma_i - \Sigma_i^{\rm proj}\| \le C_1 h^3 + C_1 h^3 + C_3 h^4 \le C_4 h^3,\] 
assuming $\eps$ is small enough that $h = A\eps \le 1$. 
By rescaling the $t_j$ and applying \lemref{Sigma_proj} below ($A$ was chosen to make things work), we find that 
\[C_5^{-1} h^2 P_i \preceq \Sigma_i^{\rm proj} \preceq C_5 h^2 P_i.\]  
When this is the case, $\Sigma_i^{\rm proj}$ has exactly $k$ nonzero eigenvalues, all between $C_5^{-1} h^2$ and $C_5 h^2$, and so by the Davis--Kahan $\sin\Theta$ theorem~\cite[Th~V.3.6]{stewart1990matrix}, 
\[\|Q_i - Q_i^{\rm proj}\| \le \frac{\sqrt{2}\, \|\Sigma_i - \Sigma_i^{\rm proj}\|}{\lambda_k(\Sigma_i^{\rm proj}) - \lambda_{k+1}(\Sigma_i^{\rm proj})} \le \frac{\sqrt{2}\, C_4 h^3}{C_5^{-1} h^2} = C_6 h,\] 
where $Q_i$ and $Q_i^{\rm proj}$ are the projections onto the top $k$-dimensional eigenspaces of $\Sigma_i$ and $\Sigma_i^{\rm proj}$ respectively. 
Since \smash{$Q_i^{\rm proj} = P_i$}, the result follows from \lemref{proj_angle} and the fact that $\sin a \ge \frac2\pi a$ for all $a \in [0,\frac\pi2]$.
\end{proof}

\begin{lemma} 
\label{lem:Sigma_proj} 
Suppose that $u_1, \dots, u_N \in \bbR^k$ is a $(\eta, 1/C_1)$-net of the unit ball. Define $\Sigma = \frac1N \sum_j u_j u_j^\top$. Then, for $\eta \le \eta_0$ for some $\eta_0 > 0$ depending only on $k$ and $C_2 \ge 1$ depending only on $k$ and $C_1$, we have $C_2^{-1} {\rm I} \preceq \Sigma \preceq C_2 {\rm I}$. 
\end{lemma}

\begin{proof}
Let $B_0$ denote the unit ball in $\bbR^k$.    We make use of Riemann sums. Let $V_j$ denote the cell corresponding to $u_j$ in the Voronoi partition of $B_0$ based on $u_1, \dots, u_N$, and define
\[\tilde\Sigma := \sum_{j\in[N]} \mu(V_j) u_j u_j^\top.\]
Using the fact that $u \mapsto u u^\top$ is 2-Lipschitz on $B_0$, we have 
\begin{equation} \label{Sigma_proj1}  
    \left\| \tilde\Sigma - \int_{B_0} u u^\top \d u \right\| 
    \le 2 \max_{j\in[N]} \diam(V_j), 
\end{equation}
where $\mu$ denotes the Lebesgue measure.
 
We now show that $\mu(V_j)$ is of order $\eta^k$ uniformly in $j$. 
Indeed, on the one hand, $\|u_j - u_l\| > \eta/C_1$, so that $B(u_j,\eta/2C_1) \subset V_j$, implying that 
\[\mu(V_j) \ge \mu\big(B(u_j,\eta/2C_1)\big) \ge \mu(B_0) (\eta/2C_1)^k =: \eta^k/C_3.\] 
On the other hand, $V_j \subset B(u_j, 2\eta)$ since any $u \in B_0$ such that $\|u - u_j\| > \eta$ must be within $\eta$ of some $u_l$ other than $u_j$, and this implies that 
\[\mu(V_j) \le \mu\big(B(u_j,2\eta)\big) \le \mu(B_0) (2\eta)^k =: C_4 \eta^k.\] 
Because $\sum_j \mu(V_j) = \mu(B_0) = 1$ and $N \min_j \mu(V_j) \le \sum_j \mu(V_j) \le N \max_j \mu(V_j)$, all this implies that $\eta^{-k}/C_4 \le N \le C_3 \eta^{-k}$.
We thus have 
\[(C_3 C_4)^{-1} \Sigma \preceq N (\min_j \mu(V_j)) \Sigma \preceq \tilde\Sigma \preceq N (\max_j \mu(V_j)) \Sigma \preceq (C_3 C_4) \Sigma.\]
Reorganized, this gives
\[(C_3 C_4)^{-1} \tilde\Sigma \preceq \Sigma \preceq (C_3 C_4) \tilde\Sigma.\]
Note that $C_3$ and $C_4$ only depend on $k$ and $C_1$.
 
Along the way, we also found that $\diam(V_j) \le 4 \eta$ due to $V_j \subset B(u_j, 2\eta)$.
And since, by symmetry, 
\[\int_{B_0} y y^\top \d y = C_5 {\rm I},\] 
for $C_5$ depending only on $k$, going back to \eqref{Sigma_proj1}, we find that the eigenvalues of $\tilde\Sigma$ are between $C_5-8\eta$ and $C_5+8\eta$.
Let $\eta_0 = C_5/16$, defined so that $C_5-8\eta_0 \ge C_5/2$ and $C_5+8\eta_0 \le 2 C_5$.
Then, for $\eta \le \eta_0$, the eigenvalues of $\Sigma$ are between $C_5/(2 C_3C_4)$ and $(2C_3C_4)C_5$.
\end{proof}

\subsubsection{Estimating the surface}
Having estimated the tangent space $T_i$ at each point $x_i \in \bX$, resulting in $\hat T_i$, we build the tangential Delaunay complex (corrected for inconsistencies), denoted $\hat\cT_{\rm c}$. 

The proof of \thmref{TDC_bound} relies on a result borrowed\,\footnote{The lower bound on the reach stated in Th~4.1 in~\cite{aamari2018stability} appears incorrect. We  followed the arguments backing that result, especially Lem~4.2, to arrive at the lower bound on the reach given here.} from~\cite{aamari2018stability}, which in words says that if the estimates for the tangent spaces are accurate enough then there is a surface with positive reach which also passes through the sample points and for which each estimated tangent space corresponds to its actual tangent space at the corresponding location.
\begin{proposition}[Th~4.1 in~\cite{aamari2018stability}]
\label{prp:tangent_interpolation}
In the present context, suppose that, for all $i \in [n]$, $\tilde T_i$ is a $k$-dimensional affine subspace passing through $x_i$ such that $\angle(T_i, \tilde T_i) \le \theta$. There is a constant $C$ depending on $\cM$ such that, if $\eps \le 1/C$ and $\theta \le 1/C$, then there is a surface $\tilde\cM$ satisfying \aspref{M} with reach $\ge 1/C$ and within Hausdorff distance $C\theta\eps$ from $\cM$ such that $x_i \in \tilde\cM$ and $T_{\tilde\cM}(x_i) = \tilde T_i$ for all $i \in [n]$.
\end{proposition}

In its original statement, \cite[Th~4.1]{aamari2018stability} also says that $\cM$ and $\tilde\cM$ are diffeomorphic, and in fact, a look at the proof of that result, in particular from \cite[Lem~4.2]{aamari2018stability}, reveals that they are $O(\eps)$ distortions of each other. This is not quite enough for our purposes, and therefore we do not use this part of the result. See \secref{discussion} for a longer discussion.

\subsubsection{Estimating the distances}
With $\hat\cT_{\rm c}$ at our disposal, we compute the pairwise distances on $\hat\cT_{\rm c}$ to obtain
\begin{equation} \label{dhat_delaunay}
\hat d_{ij} := \d_{\hat\cT_{\rm c}}(x_i, x_j), \quad i, j \in [n].
\end{equation}
\begin{theorem}
\label{thm:TDC_bound}
There is $C > 0$ depending only on $\cM$ such that, if $\eps \le 1/C$, then the estimator \eqref{dhat_delaunay} satisfies \eqref{distance_error}.
\end{theorem}

\begin{proof}[Proof of \thmref{TDC_bound}]
First, let $C_0$ denote the constant of \prpref{tangent_estimation} and let $C_1$ be the constant of \prpref{tangent_interpolation}. Suppose $\eps$ is small enough that $C_0\eps \le 1/C_1$, so that  \prpref{tangent_interpolation} applies to yield the existence of a surface $\tilde\cM$ satisfying \aspref{M} such that: it contains the sample points $\bX$; its tangent spaces at the sample points coincide with the estimated tangent spaces, i.e., $T_{\tilde\cM}(x_i) = \hat T_i$ for all $i \in [n]$; it has reach $\ge 1/C_1$; and it is within Hausdorff distance $C_1 \eps^2$. 
By \lemref{both_projections}, we have
\begin{equation} 
|\d_{\cM}(x_i,x_j) - \d_{\tilde\cM}(x_i,x_j)| 
\le C_2 \eps^2\, \d_{\cM}(x_i,x_j), \quad \forall i,j \in [n].
\end{equation}

Next, suppose $\eps$ is small enough that \thmref{TDC} applies, so that $\hat\cT_{\rm c}$ and the surface $\tilde\cM$ are in one-to-one correspondence via a $C_3 \eps^2$-distortion map. 
By \corref{distortion}, this implies that 
\begin{equation} 
|\d_{\tilde\cM}(x_i,x_j) - \d_{\hat\cT_{\rm c}}(x_i,x_j)| 
\le C_4 \eps^2\, \d_{\hat\cT_{\rm c}}(x_i,x_j), \quad \forall i,j \in [n].
\end{equation}

Combining these two bounds using the triangle inequality, we get for any pair $i,j \in [n]$,
\begin{equation} 
|\d_{\cM}(x_i,x_j) - \d_{\hat\cT_{\rm c}}(x_i,x_j)| 
\le C_5 \eps^2\, \big(\d_{\cM}(x_i,x_j) + \d_{\hat\cT_{\rm c}}(x_i,x_j)\big),
\end{equation}
and, if $\eps$ is small enough that $C_5 \eps^2 \le 1/3$, this implies that
\begin{equation} 
|\d_{\cM}(x_i,x_j) - \d_{\hat\cT_{\rm c}}(x_i,x_j)| 
\le C_6 \eps^2\, \min\big\{\d_{\cM}(x_i,x_j),  \d_{\hat\cT_{\rm c}}(x_i,x_j)\big\},
\end{equation}
which gives the desired bound.
\end{proof}

\subsection{Numerical experiments}
\label{sec:numerics}

We performed some numerical experiments to illustrate our theory. We focused on the most interesting case accessible to computations, that of points on a $(k=2)$-dimensional surface embedded in a Euclidean space of dimension $d=3$. We chose to work with such emblematic surfaces as the sphere, the torus, and the Swiss roll (even though the latter has a boundary).

\subsubsection{Data}
Armed with a parameterization of a surface, we generate sample points by drawing from the uniform distribution on the parameter domain independently $n$ times, $n$ being the desired sample size (which varies in our experiments). We then subsampled the points to obtain a net. For simplicity, this subset of points was considered to be the entire sample. 


These are the parameterizations that we used:
\begin{align}
\text{sphere:} & \qquad (u,v) \in [0,2\pi) \times [0,\pi) \mapsto (\cos(u) \cos(v), \sin(u) \cos(v), \sin(v)); \\
\text{torus:} & \qquad (u,v) \in [0,2\pi) \times [0,2\pi) \mapsto ( \cos(u) (2+\cos v), \sin(u) (2+\cos v), \sin v); \\
\text{Swiss roll:} & \qquad (u,v) \in [\pi/4, 9\pi/4] \times [0, 1] \mapsto (u\cos u, u\sin u, v).
\end{align}

\subsubsection{Mesh construction}

The generation a mesh from the point cloud, specifically, the tangential Delaunay complex, was done via the implementation available in the {\em Geometry Understanding in Higher Dimensions (GUDHI)} library~\cite{gudhi:TangentialComplex}.
The main parameters are the maximum perturbation radius, which is a constraint on the amount that points may be perturbed in an effort to reduce inconsistencies, and the maximum
squared edge length of a simplex.
With a dataset that has undergone preprocessing to yield a net, the parameter values do not significantly alter the resulting mesh.
The maximum squared edge length is the most crucial parameter to adjust when some areas of the surface are poorly sampled. It was set using a priori knowledge of the true underlying surface.

In \figref{mesh_examples}, we provide examples of tangential Delaunay complex mesh constructions for the sphere, the torus, and the Swiss roll, and do so for various sample sizes. 

\begin{figure}[htb!]
\includegraphics[width=.33\linewidth]{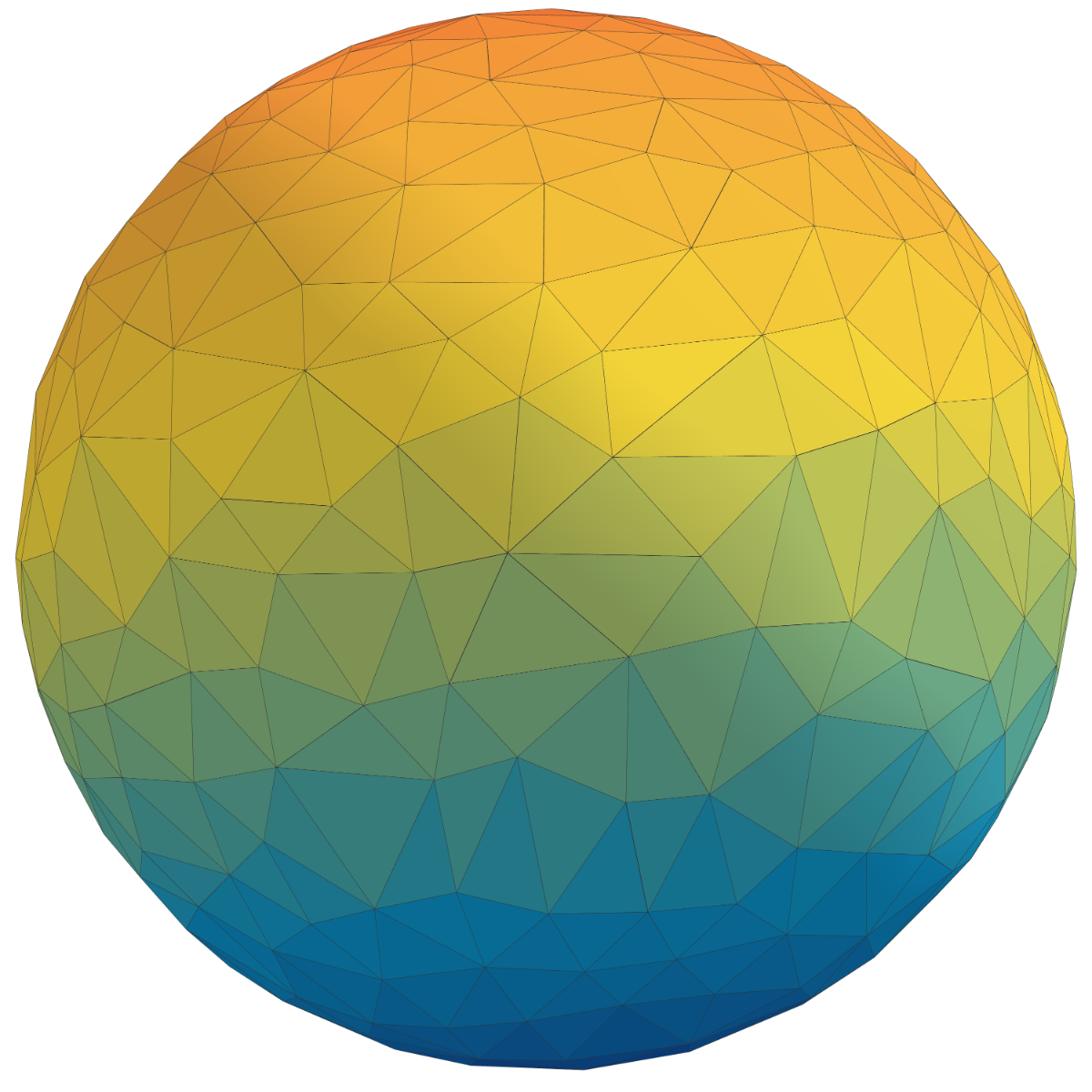} 
\includegraphics[width=.33\linewidth]{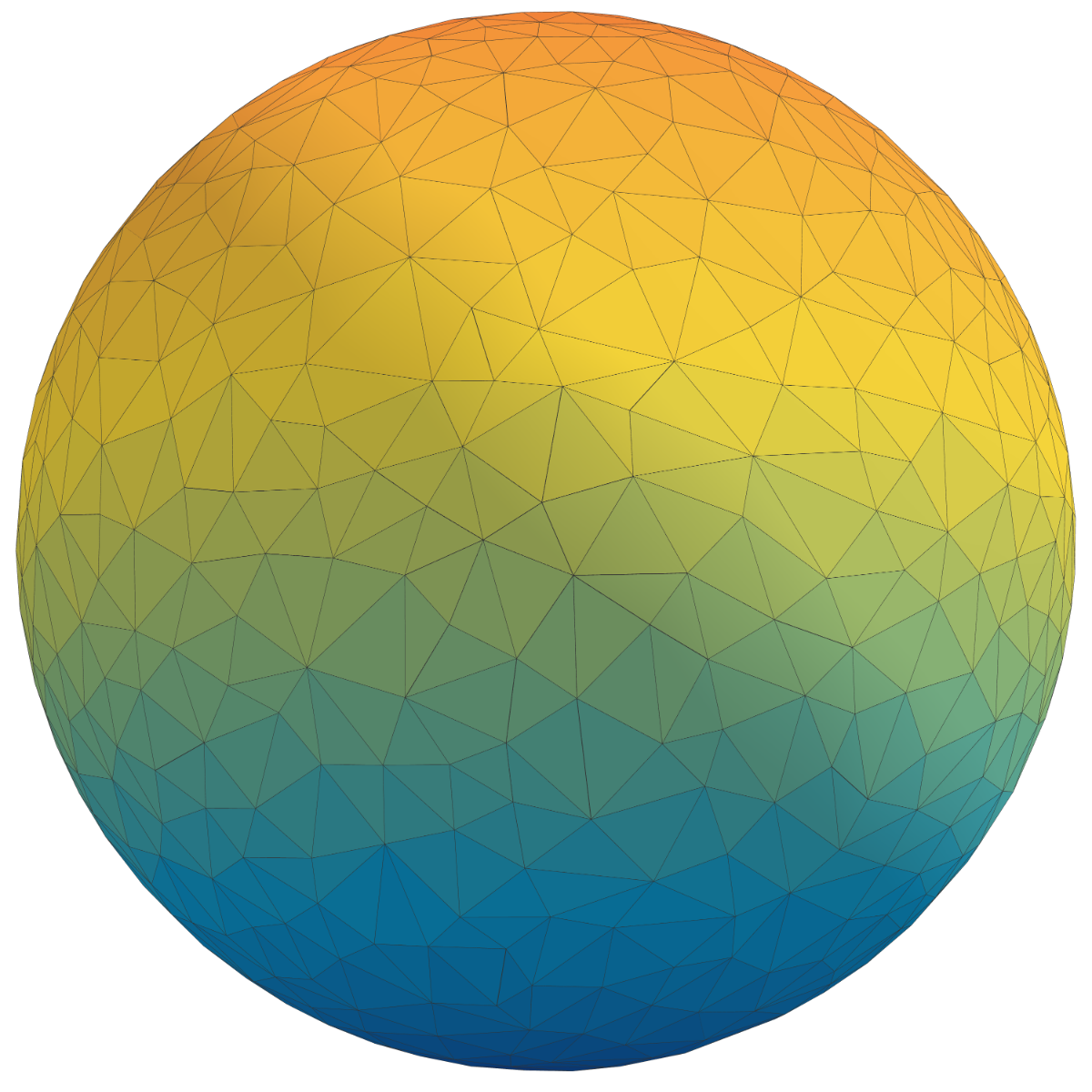} 
\includegraphics[width=.33\linewidth]{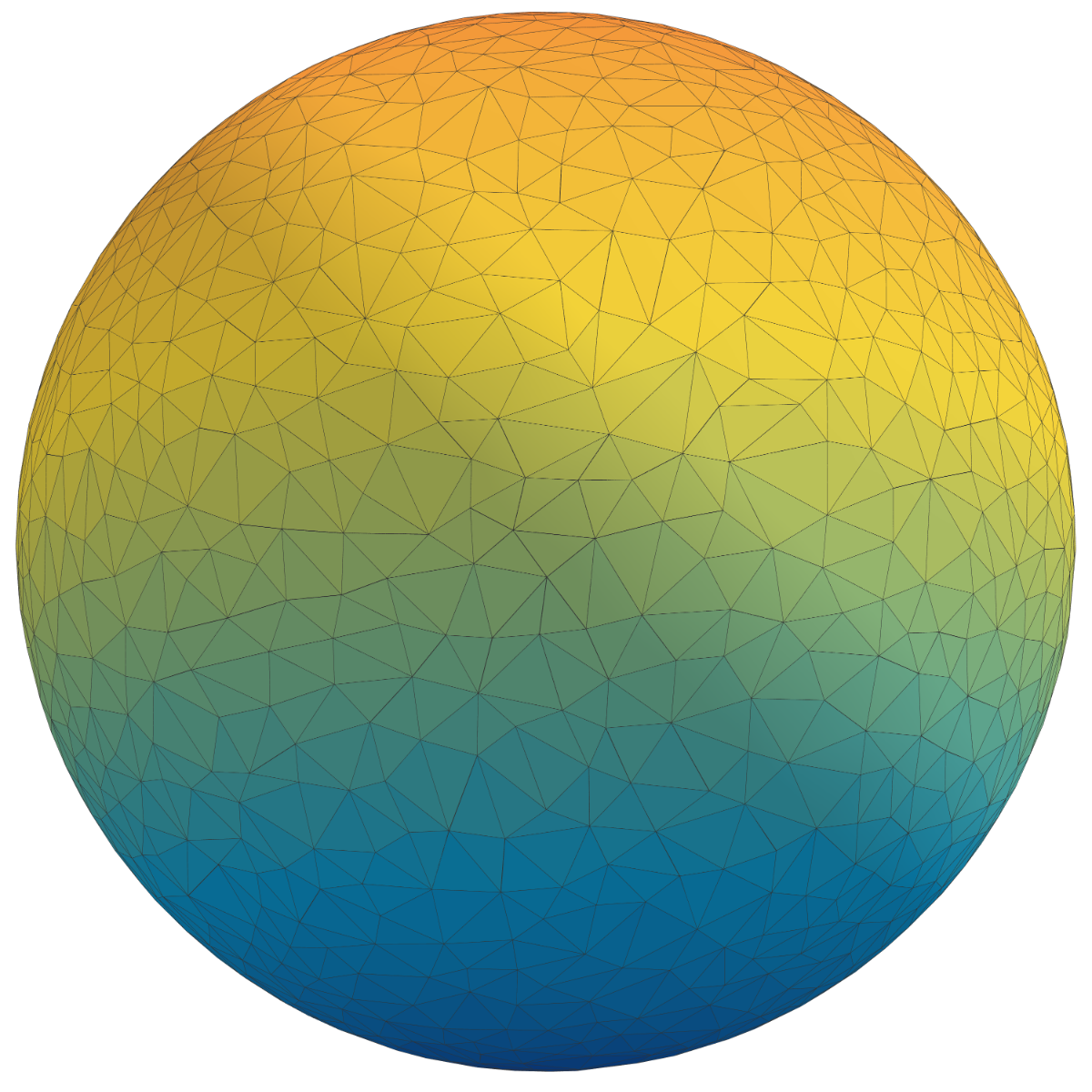} \\ 
    \includegraphics[width=.33\linewidth]{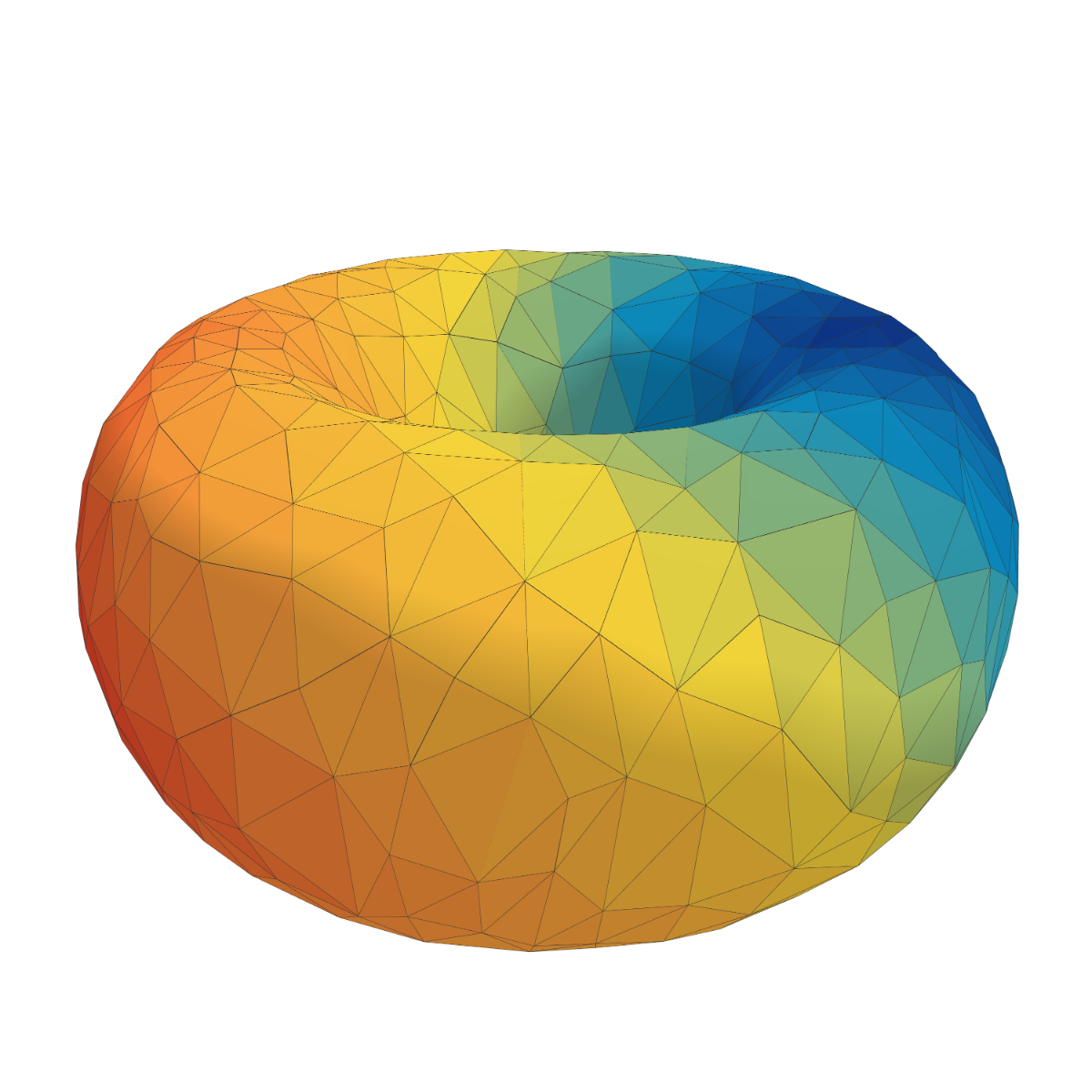} 
\includegraphics[width=.33\linewidth]{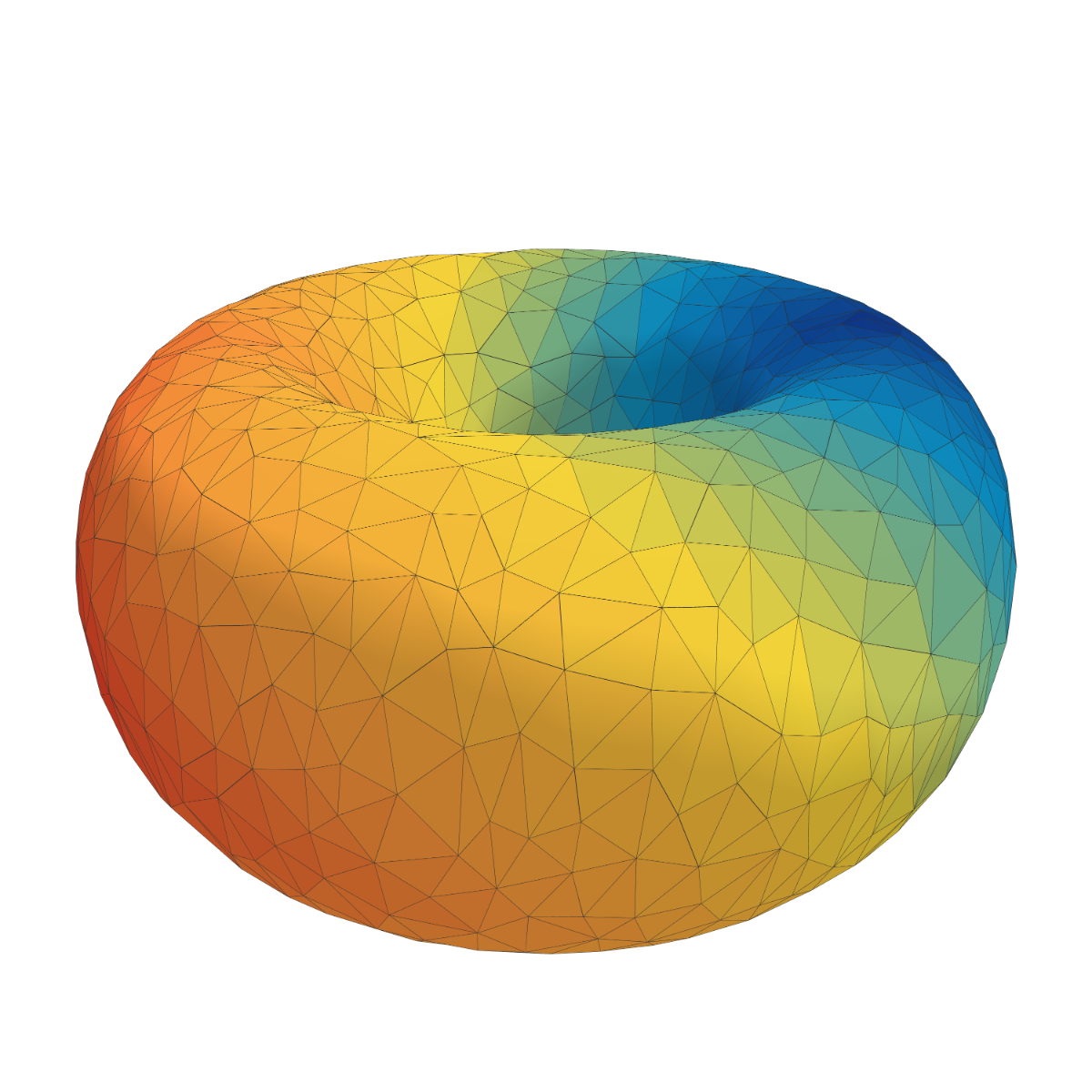} 
\includegraphics[width=.33\linewidth]{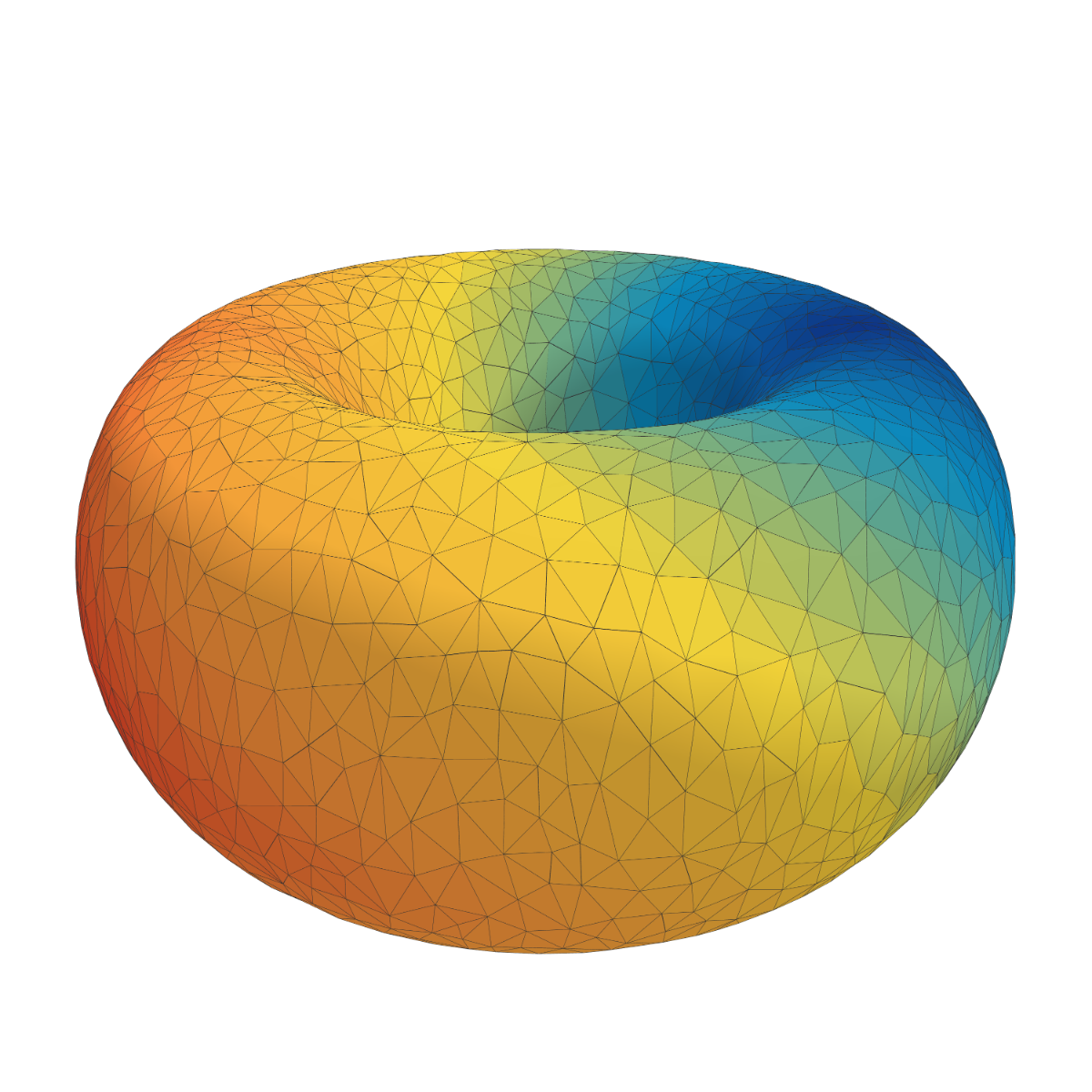} \\ 
\includegraphics[width=.33\linewidth]{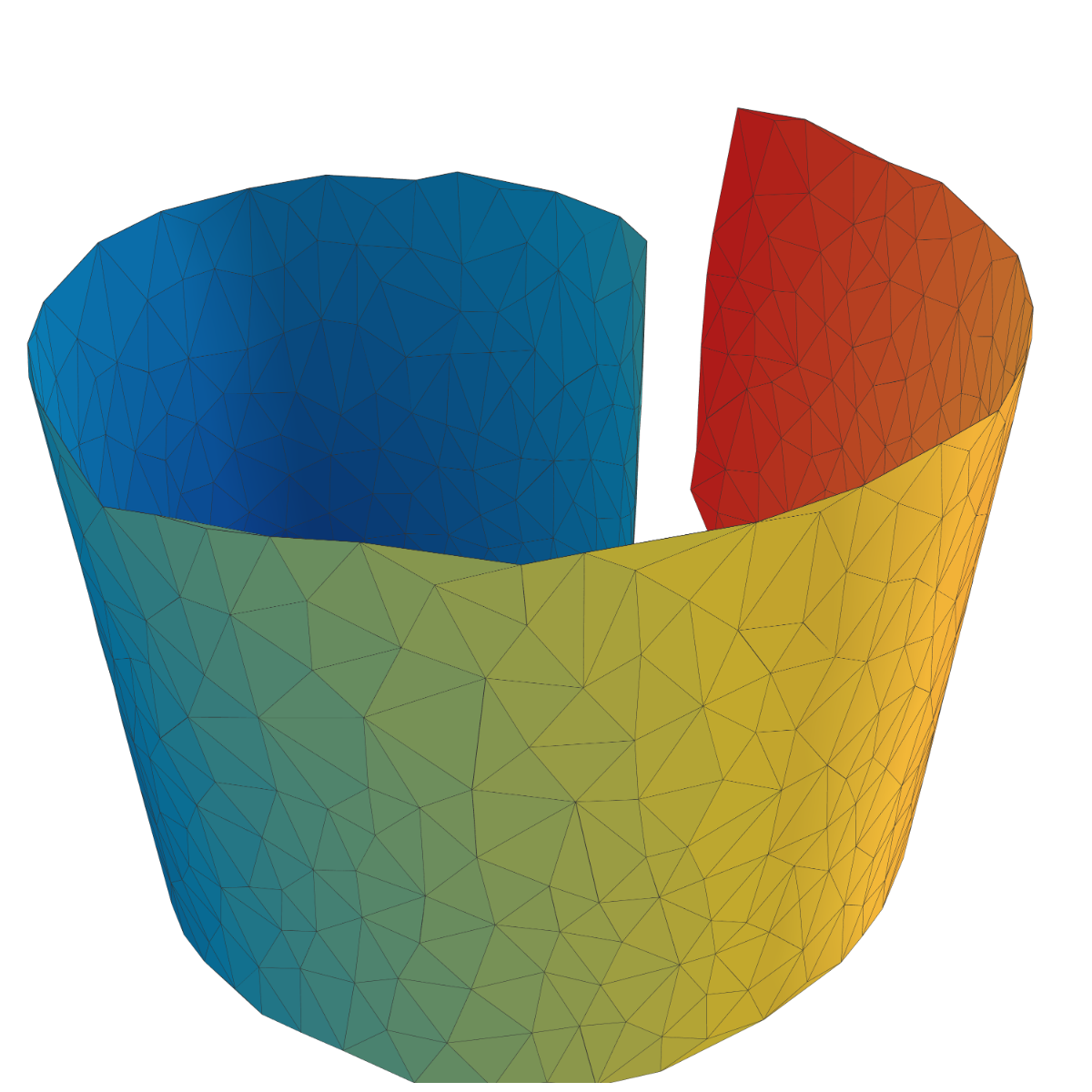} 
\includegraphics[width=.33\linewidth]{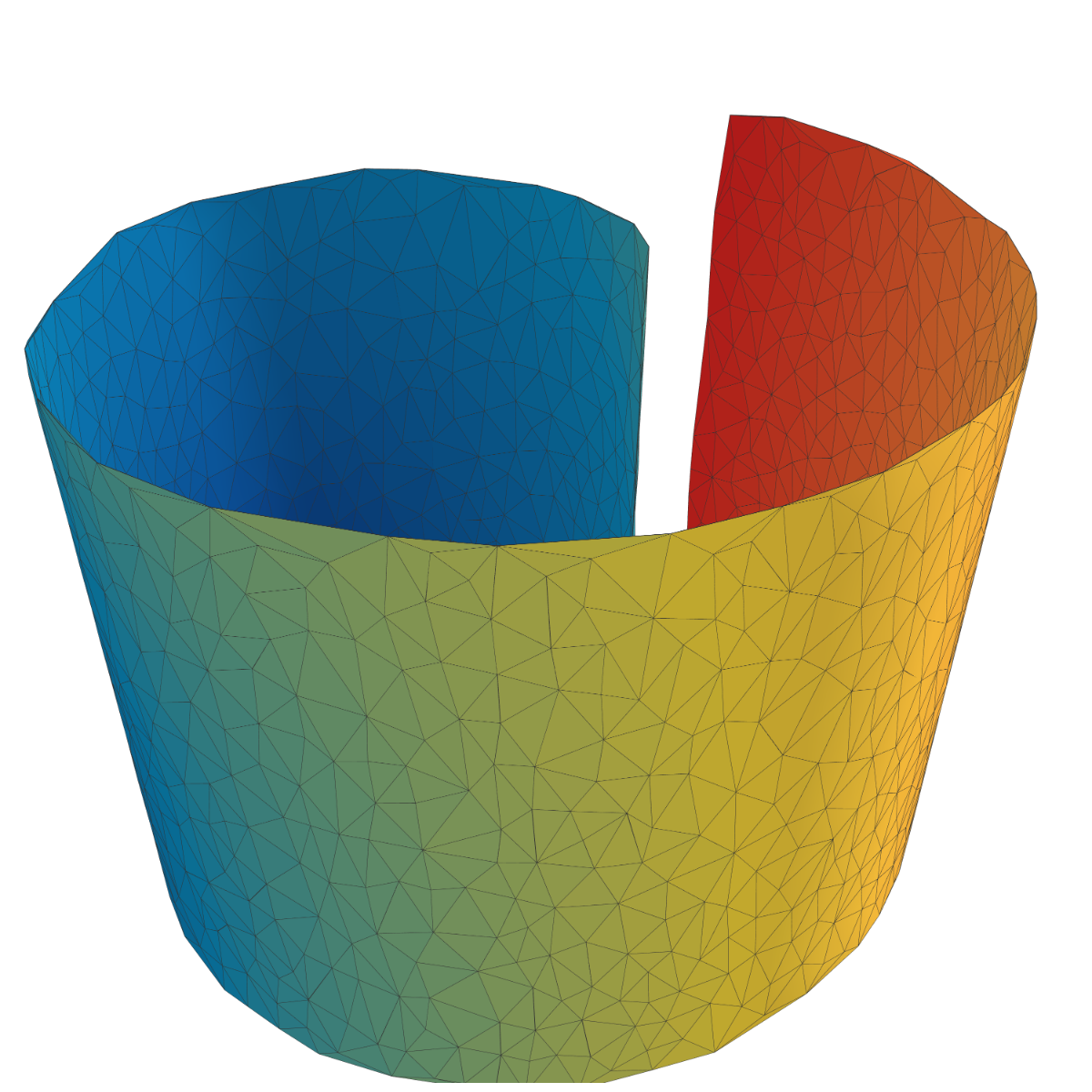} 
\includegraphics[width=.33\linewidth]{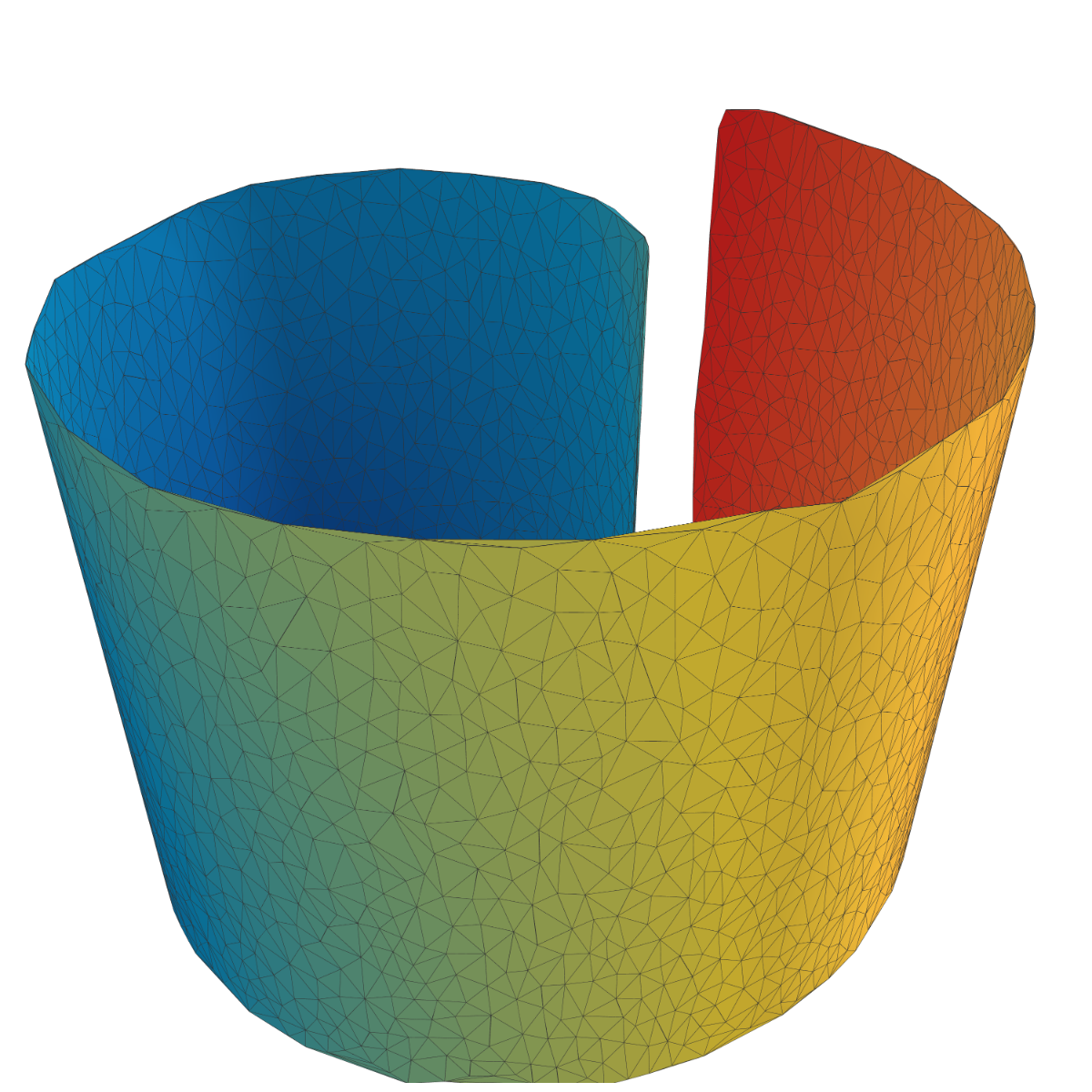} 
\caption{Examples of tangential Delaunay complex mesh constructions on the sphere, the torus, and the Swiss roll (top to bottom) based on uniform samples of size 500, 1000, and 2000 (left to right). The color signifies the distance to a fixed source point, from closer (blue) to farther (red).} 
\label{fig:mesh_examples}
\end{figure}

\subsubsection{Shortest paths}
We now turn to evaluating the accuracy of the proposed method for estimating the pairwise distances on a surface. As benchmark, we use
Isomap~\cite{tenenbaum2000global}. Isomap estimates pairwise distances on the underlying surface by forming a neighborhood graph and then computing the shortest-path distances using Djikstra's algorithm.\footnote{The igraph package was used to compute shortest paths on graphs~\cite{igraphpackage}.} The accuracy of this estimation depends crucially on the connectivity radius $r>0$ which defines the neighborhood graph. As there is no standard data-driven way to choose this connectivity radius, in our experiment, we look at various choices in a reasonable range.

On the other hand, after meshing --- which, it is true, may require some tuning to achieve a reasonable reconstruction --- computing distances on the mesh does not require further tuning.
In our experiments, we used the the `triangulated surface mesh shortest paths' module of the {\em Computational Geometric Algorithms Library (CGAL)}, which implements a variation of the Chen--Han algorithm~\cite{cgal:sp,Xin2009,Chen1990}.

For the sphere $\mathbb{S}^2$, the great circle distance (i.e., intrinsic distance) is given by $\cos^{-1} \langle x,y\rangle$ between $x,y\in \mathbb{S}^2$.
The Swiss roll is another nice surface to work with because there is a global isometry
between the surface and a rectangle in $\bbR^2$.
Let $\gamma(t) = (\gamma_1(t), \gamma_2(t)) := (t\cos(\alpha t), t\sin(\alpha t))$.
The arc length is given by 
\[s(t) = \int_0^t \|\dot\gamma(t)\| {\rm d}t = \int_0^t \Big[\frac{t}{2} \sqrt{1+(\alpha t)^2} + \frac{1}{2\alpha}\sinh^{-1}(\alpha t)\Big] {\rm d}t.\]
The inverse $t = t(s)$ is calculated with Newton's method.
By construction, $\phi(s, z) := (\gamma_1(s), \gamma_2(s), z)$ is an isometry, and in particular, 
\[d_\cM(\phi(s_1,z_1), \phi(s_2,z_2)) = \sqrt{(s_1-s_2)^2 + (z_1-z_2)^2}.\]
By contrast, the torus is not as easy to handle. It does admit a $C^1$ isometry into $\bbR^3$, but not a $C^2$ isometry, and the $C^1$ isometry is rather complex \citep{borrelli2013isometric}.
We opted for a numerical approximation ascertained by the use of a midpoint method initialized with the output from a neighborhood graph distances.
The method is iterative. From an existing approximating path, say $\gamma^{(t)}$ at iteration $t$, a new path is constructed at iteration $t+1$ by first splitting each line segment of $\gamma^{(t)}$ in half and then projecting the resulting piecewise linear path back onto the surface.
The results of these experiments are reported in \figref{exp_sphere} (sphere), \figref{exp_torus} (torus), and \figref{exp_swiss} (Swiss roll).
For the sphere and Swiss roll, mesh distances are noticeably more accurate on average than graph distances, and so across a wide range of choices of the connectivity radius of the graph. For the torus, the mesh distances are only slightly better than graph distances corresponding to the best choice of connectivity radius. (We again note that the choice of radius is typically done in a rather ad hoc manner in practice.)

\begin{figure}[!htb]
\centering 
\includegraphics[width=.30\linewidth]{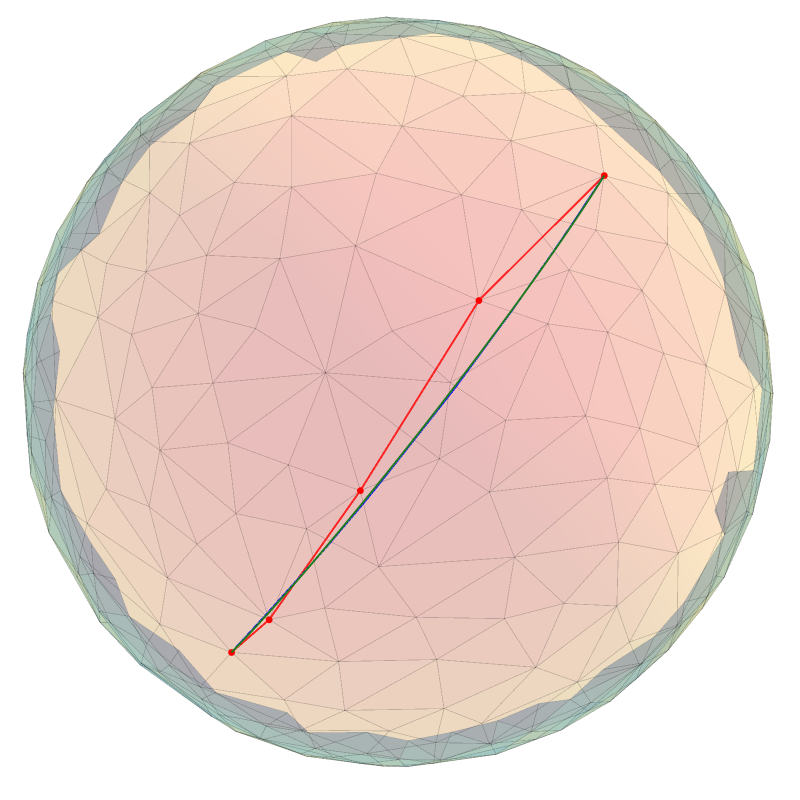} \quad
\includegraphics[width=.30\linewidth]{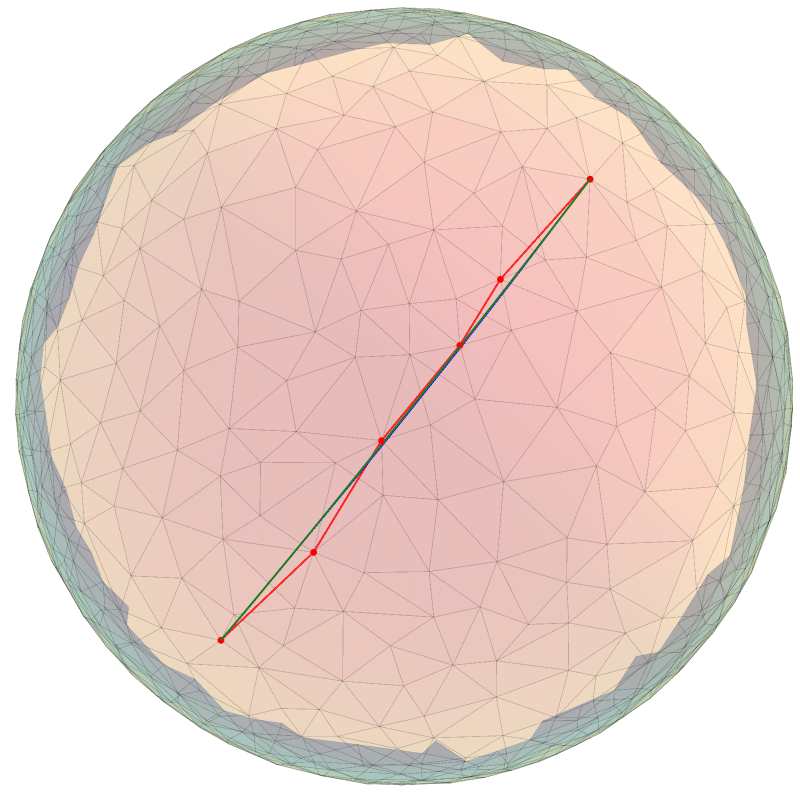} \quad
\includegraphics[width=.30\linewidth]{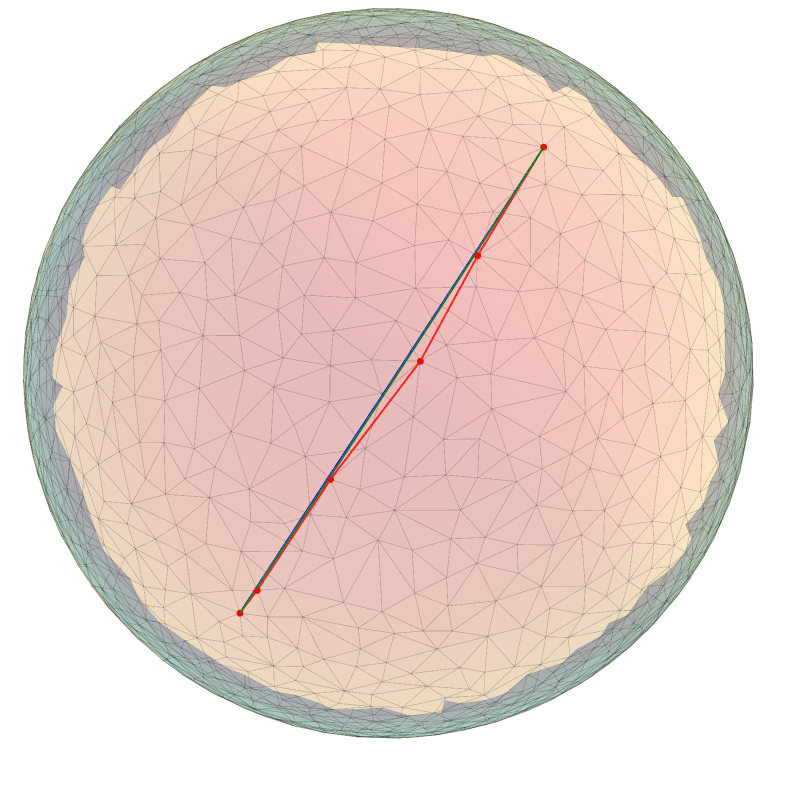} 
\includegraphics[width=.30\linewidth, trim=35 0 50 45, clip]{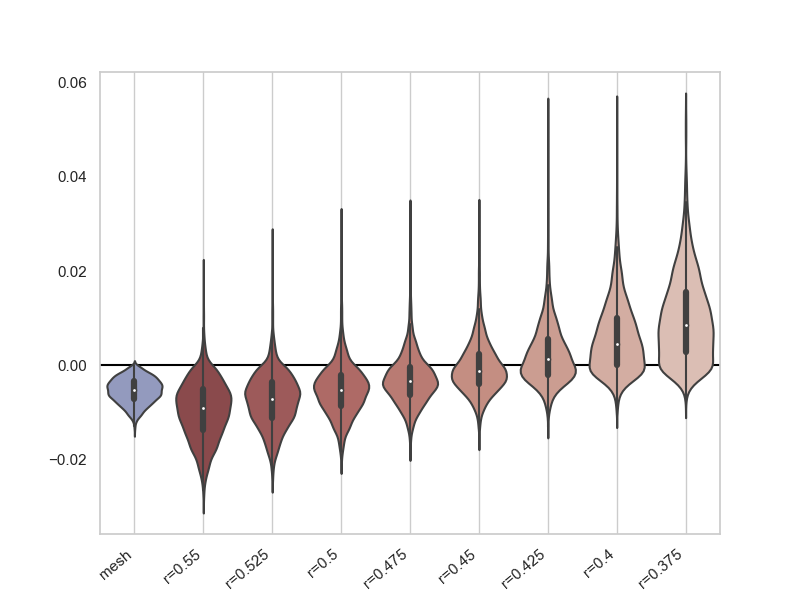} \quad
\includegraphics[width=.30\linewidth, trim=35 0 50 45, clip]{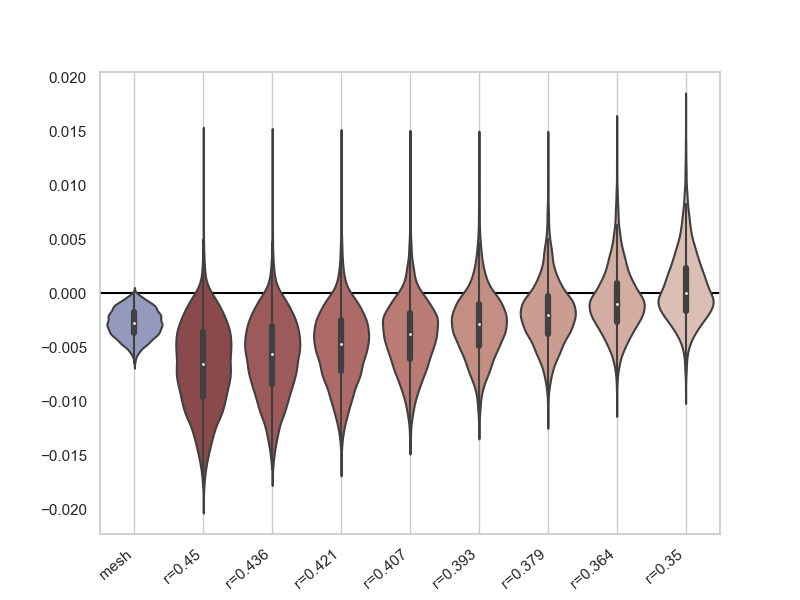} \quad
\includegraphics[width=.30\linewidth, trim=35 0 50 45, clip]{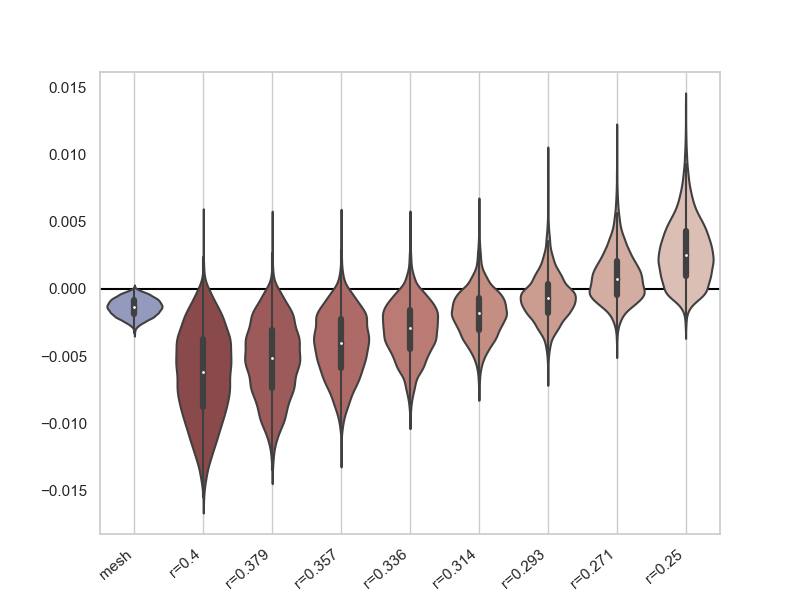} 
\caption{Estimation of distances for the sphere. Sample size $n = 500, 1000, 2000$ (left to right). 
Top row: Examples of computed shortest paths comparing the true path (green), path computed on the mesh (blue), and a typical path computed on the neighborhood graph (red). The first two paths almost overlap. 
Bottom row: Signed error (estimate - true) averaged over 50 repeats for the distance computed on the mesh and for the distance computed on a neighborhood graph of varying connectivity radius $r$. 
} 
\label{fig:exp_sphere}
\end{figure}

\begin{figure}[!htb] 
\centering
\includegraphics[width=.30\linewidth, trim=30 30 30 30, clip]{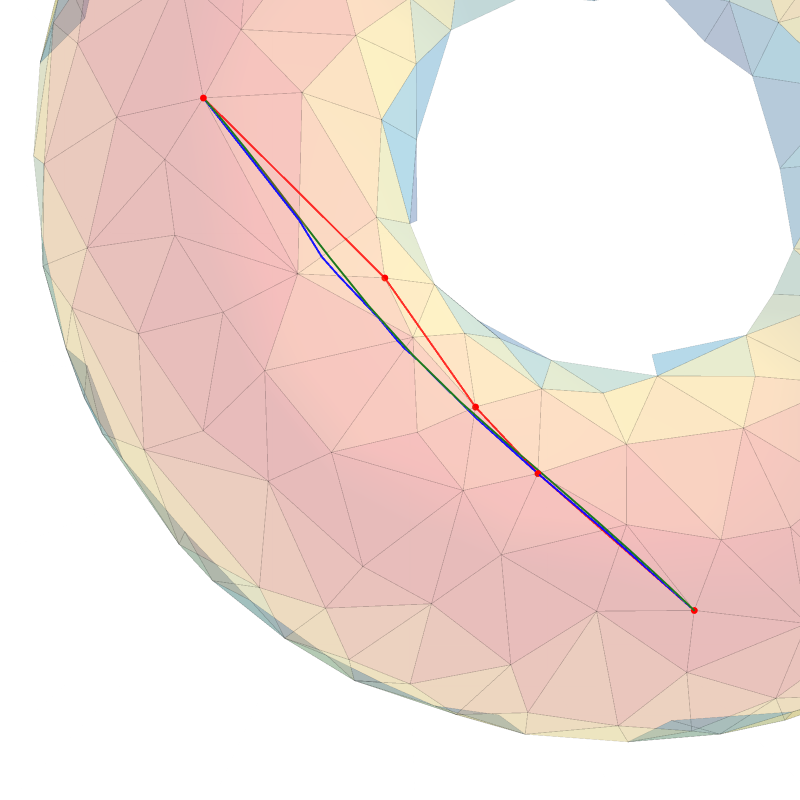} \quad
\includegraphics[width=.30\linewidth, trim=30 30 30 30, clip]{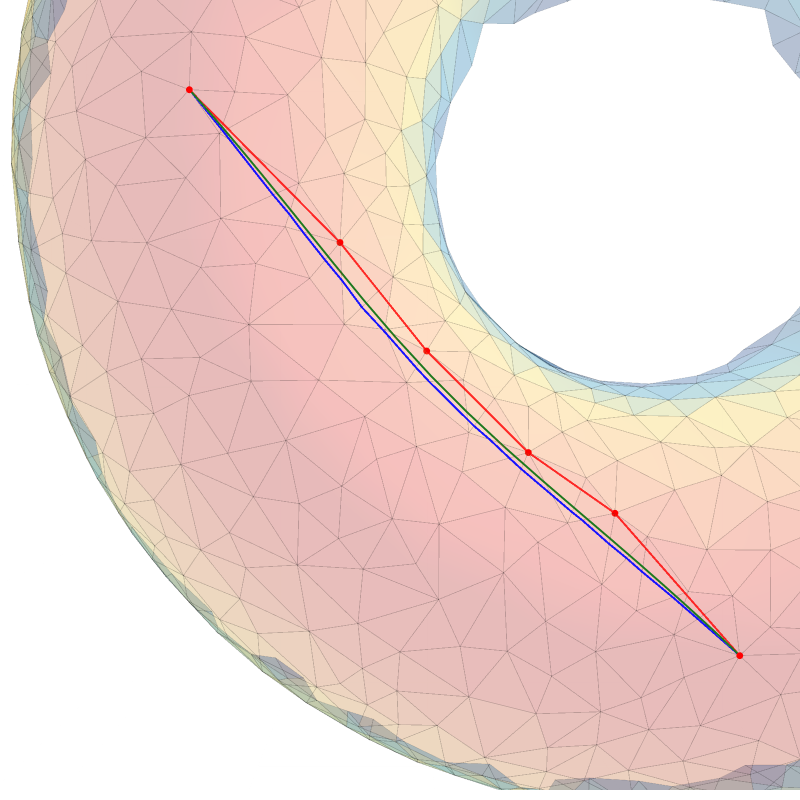} \quad
\includegraphics[width=.30\linewidth, trim=30 30 30 30, clip]{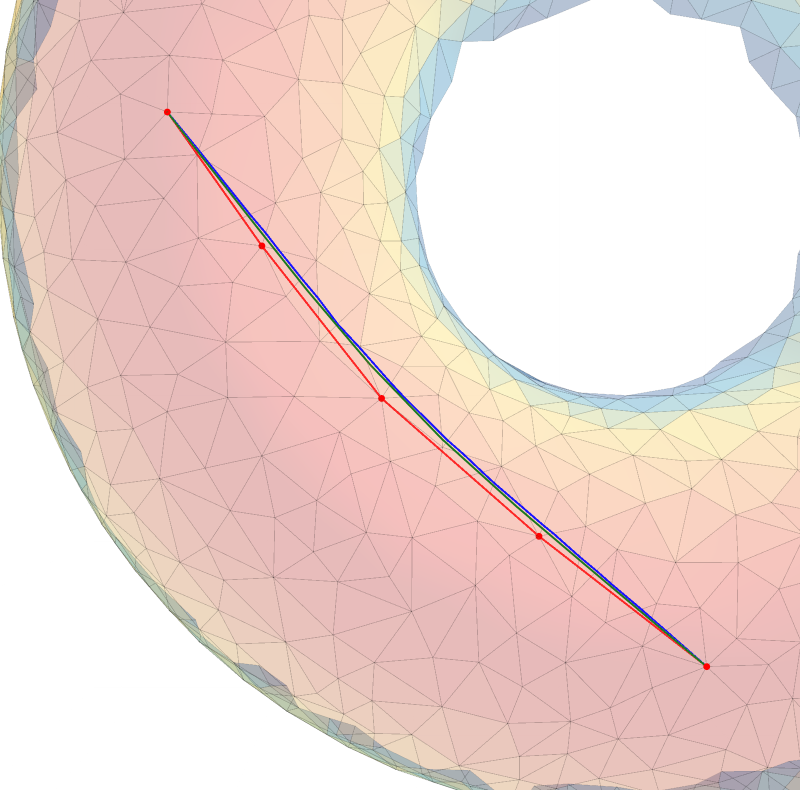} 
\includegraphics[width=.30\linewidth, trim=35 0 50 45, clip]{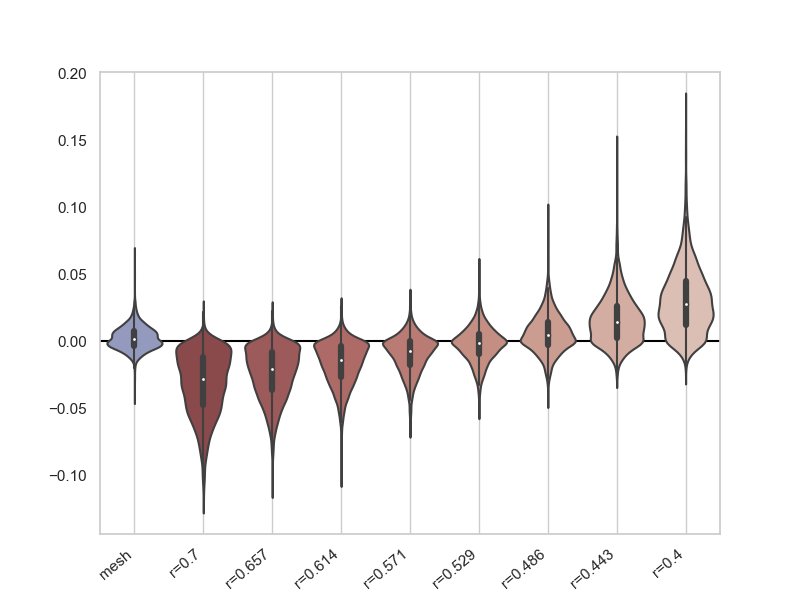} \quad
\includegraphics[width=.30\linewidth, trim=35 0 50 45, clip]{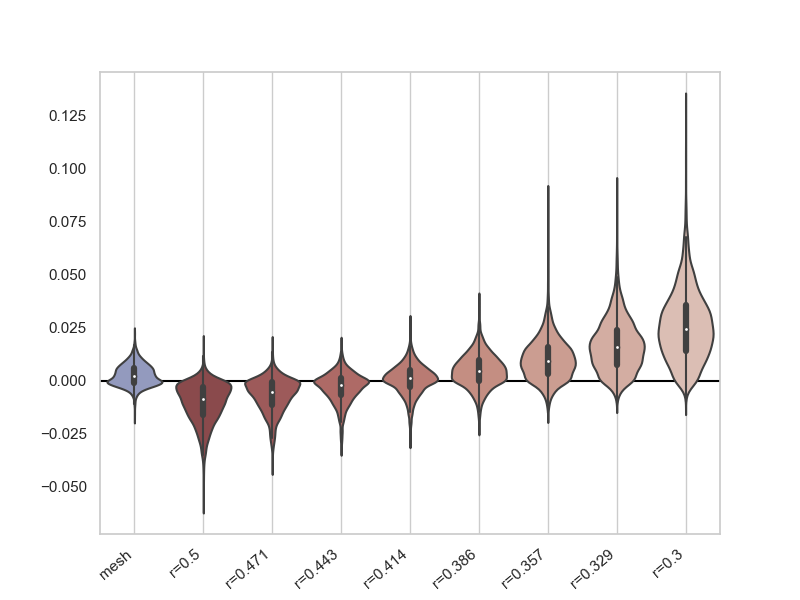} \quad
\includegraphics[width=.30\linewidth, trim=35 0 50 45, clip]{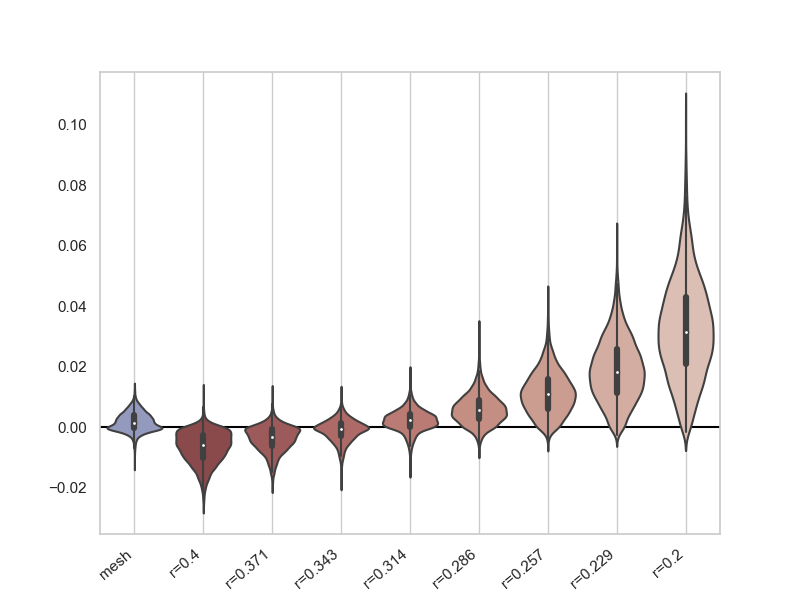} 
\caption{Estimation of distances for the torus. (See \figref{exp_sphere} for details.)} 
\label{fig:exp_torus}
\end{figure}

\begin{figure}[!htb]
\centering 
\includegraphics[width=.30\linewidth, trim=10 10 10 10, clip]{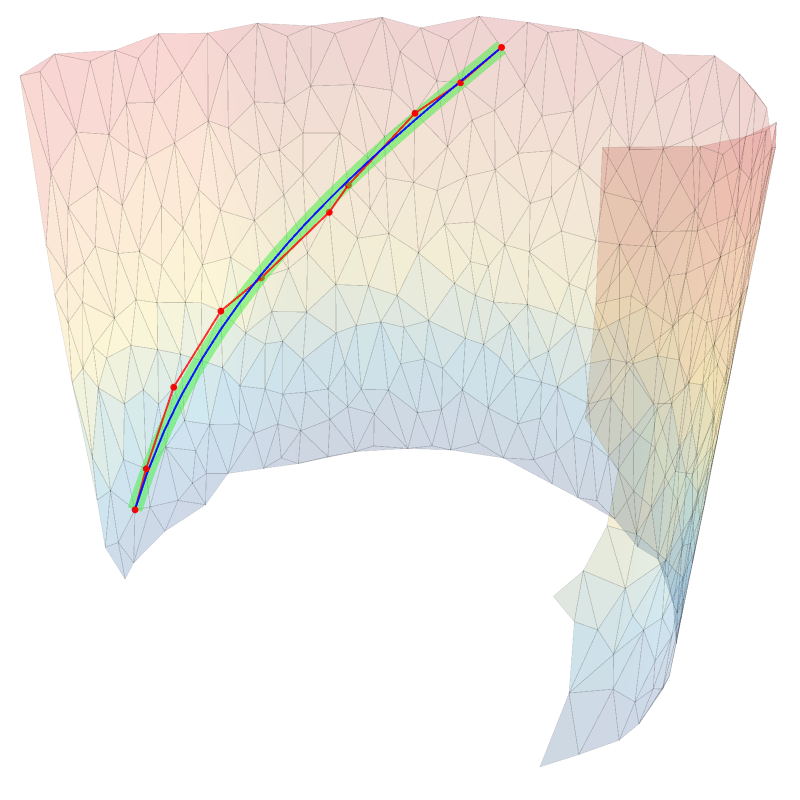} \quad
\includegraphics[width=.30\linewidth, trim=10 10 10 10, clip]{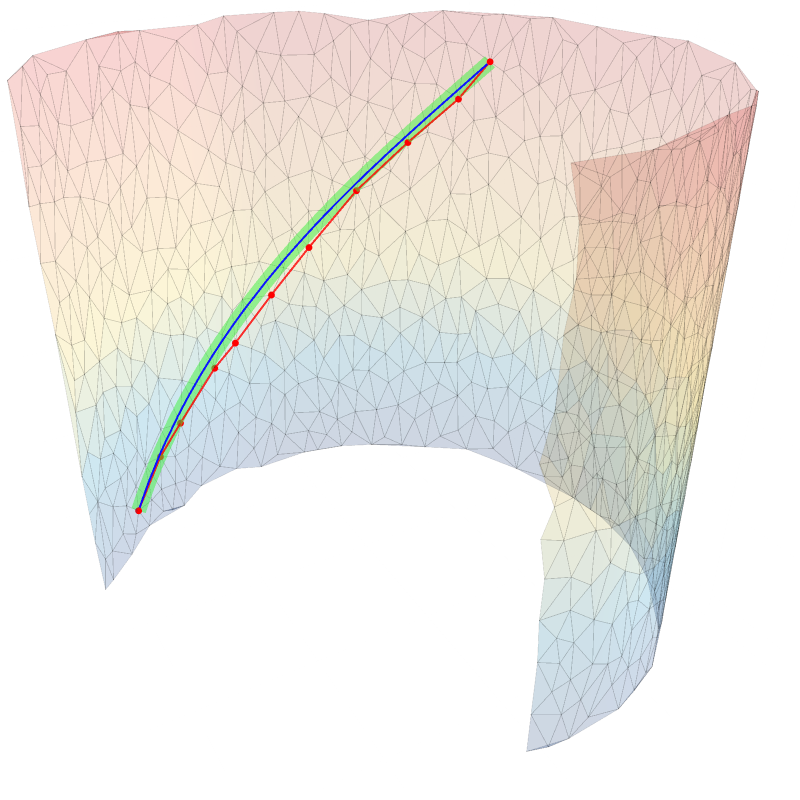} \quad
\includegraphics[width=.30\linewidth, trim=10 10 10 10, clip]{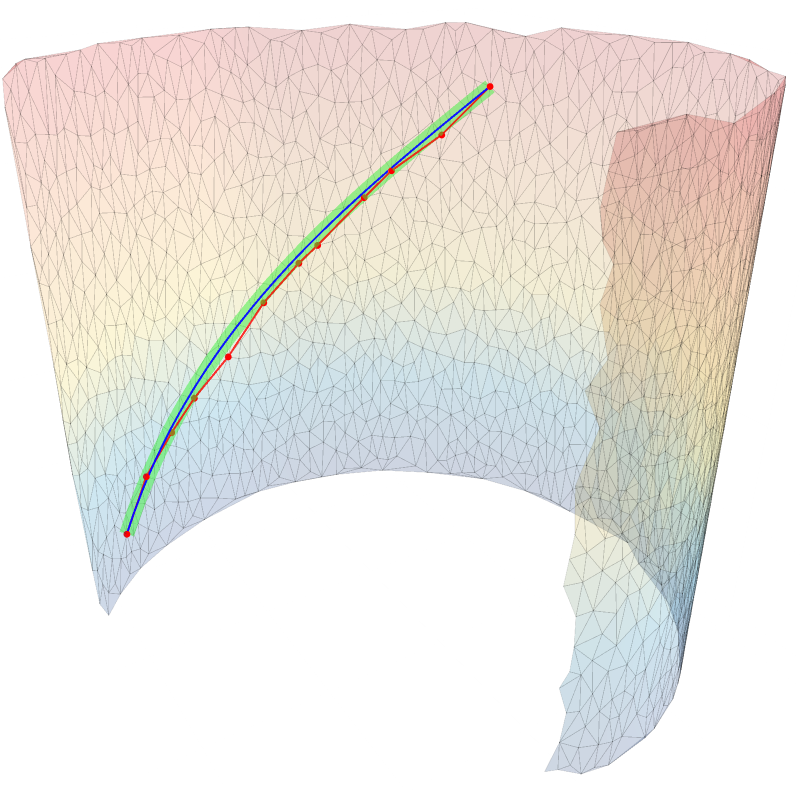} 
\includegraphics[width=.30\linewidth, trim=35 0 50 45, clip]{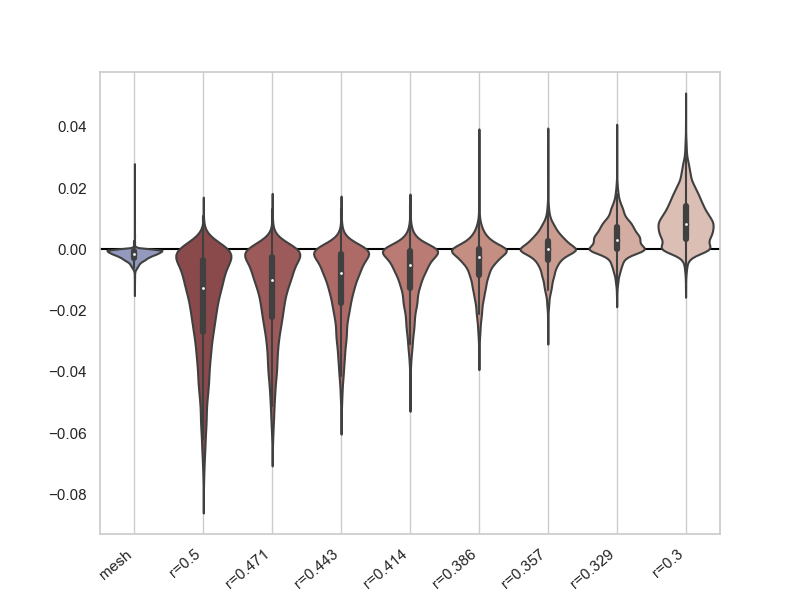} \quad
\includegraphics[width=.30\linewidth, trim=35 0 50 45, clip]{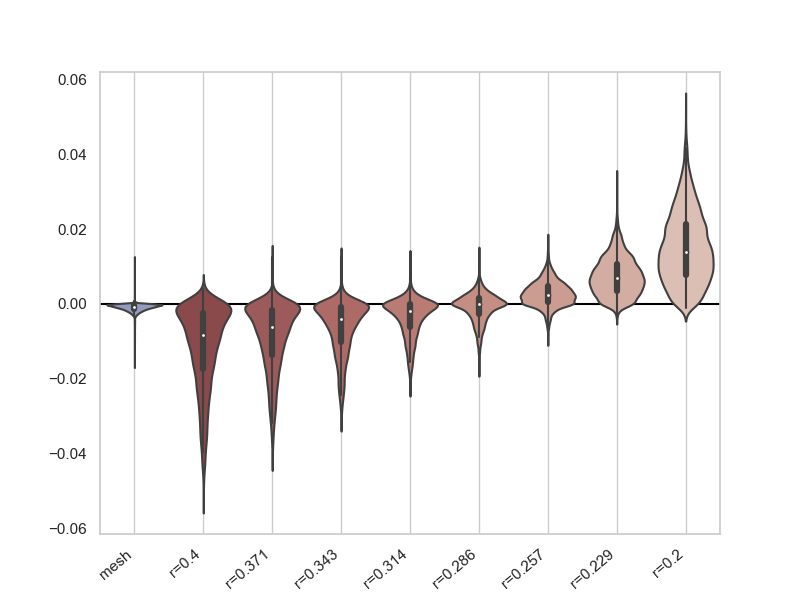} \quad
\includegraphics[width=.30\linewidth, trim=35 0 50 45, clip]{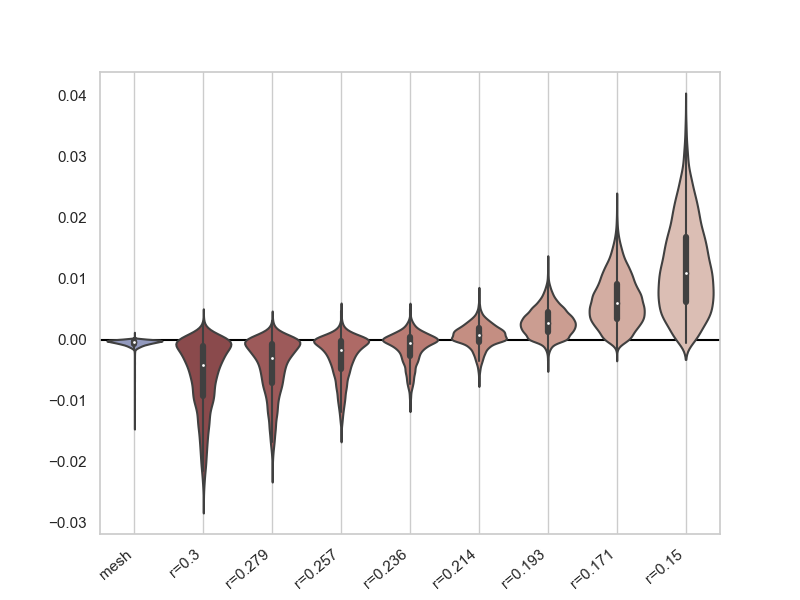} 
\caption{Estimation of distances for the Swiss roll. (See \figref{exp_sphere} for details.)} 
\label{fig:exp_swiss}
\end{figure}

We also examined the accuracy of the two methods as a function of the sample size.
For this experiment, we focused on the sphere.
As we explored larger sample sizes, we approximate the error by evaluating the difference between true and estimated distance on 100 pairs of points chosen at random.
The result of this experiment is reported in \figref{varying}.

\begin{figure}[!htb] 
\centering 
\includegraphics[width=.5\linewidth]{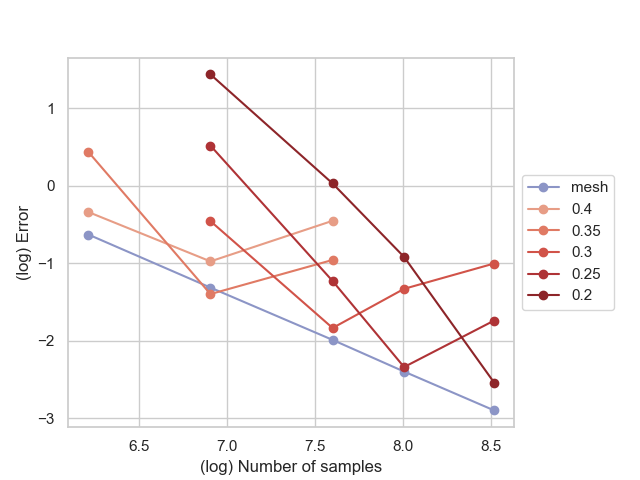} 
\caption{Estimation of distances on the sphere. The setting is as in \figref{exp_sphere}, except that here we look at a wider range of sample sizes. (Also, the error is absolute, not signed.) The values of the neighborhood radius --- each with a different color specified in the side legend box --- were chosen so as to optimize the accuracy of corresponding distance estimation.} 
\label{fig:varying}
\end{figure}

\section{Minimax manifold learning}
\label{sec:isomap}

The modern era of manifold learning, aka (nonlinear) dimensionality reduction, may have started with the advent of Isomap~\cite{tenenbaum1997mapping, tenenbaum2000global} and Local Linear Embedding (LLE)~\cite{Roweis00LLE}. This led to a flurry of methods, including Laplacian Eigenmaps~\citep{Belkin03}, Manifold Charting~\citep{brand2003charting}, Diffusion Maps~\citep{coifman2006diffusion}, Hessian Eigenmaps (HLLE)~\citep{Donoho03Hessian}, Local Tangent Space Alignment (LTSA)~\citep{Zhang04LTSA}, Maximum Variance Unfolding (aka Semidefinite Embedding)~\citep{mvu04}, $t$-SNE~\cite{maaten2008visualizing}, and UMAP~\cite{mcinnes2018umap}, among others.

Some theory was developed for many of them, in the original article or in followup publications such as~\cite{bernstein2000graph, zha2007continuum, arias2013convergence, ye2015discrete, gine2006empirical, smith2008convergence, belkin2008towards, vonLuxburg08, singer2006graph, hein2005graphs, goldberg2008manifold, arias2020perturbation}.
To this day, however, there is no optimality theory of manifold learning --- at least as far as we know. In fact, there is no clear agreement on what manifold learning is all about. We focus here on what we believe to be the simplest, and arguably the most fundamental framework for manifold learning: recovering a global isometry when one exists. Thus we assume that the underlying surface $\cM \subset \bbR^d$ is isometric to a (compact) domain $\cU \subset \bbR^k$ of full dimension (i.e., with non-empty interior). If $\phi: \cU \to \cM$ is such an isometry, then the goal is to estimate the embedded points $u_i := \phi^{-1}(x_i)$, up to a rigid transformation. Remember that $\bX = \{x_1, \dots, x_n\}$ denotes the sample and is assumed to belong to $\cM$. We denote $\bU := \{u_1, \dots, u_n\}$.
Note that, because of the isometric correspondence, since $\bX$ is an $\eps$-covering of $\cM$, $\bU$ is an $\eps$-covering of $\cU$.


Because the domain $\cU$ has a boundary, so does $\cM$. To keep the exposition simple, and to enable Isomap and the variant we propose to be consistent, we assume that $\cU$ is convex.
To be specific, we assume the following.
\begin{assumption} 
\label{asp:M_isomap} 
There is a compact and convex domain with non-empty interior in $\bbR^k$ and a $C^2$ isometry defined on an open set containing that domain such that $\cM$ is the image of that domain via that isometry.
\end{assumption}

In our context, two embeddings are necessarily compared up to a rigid transformation. We are able to leverage the results from \secref{distortion} to establish the existence of an embedding procedure that returns $\hat u_1, \dots, \hat u_n$ with error bounded as follows
\begin{equation} \label{embed_error}
\min_{Q \in \bbQ_k}\, \max_{i \in [n]}\, \|\hat u_i - Q(u_i)\| \le C \eps^2,
\end{equation} 
where $\bbQ_k$ denotes the class of rigid transformations of $\bbR^k$, and $C$ is a constant that depends on $\cM$.

\subsection{Embedding by surface reconstruction} 
\label{sec:embedding}
As we noted in \remref{boundary}, the derivations and conclusions of \secref{surface_reconstruction} can be extended to surfaces with `nice-enough' boundary, which is certainly the case for surfaces that satisfy \aspref{M_isomap}. 
Following what we did in that section, let $\bbM = \bbM(k, \bX, \eps)$ denote the class of surfaces satisfying \aspref{M_isomap} for which $\bX$ is an $\eps$-covering. We know that $\bbM$ is non-empty since $\cM \in \bbM$. 
Let $\rho_{\rm max}$ denote the supremum reach among surfaces in $\bbM$, and let $\alpha_{\rm max}$ denote the maximum $\alpha$ (defined in \remref{boundary}) of a surface in $\bbM$ with reach $\ge \rho_{\rm max}/2$.
Finally, select any surface $\hat\cM \in \bbM$ satisfying $\rho(\hat\cM) \ge \rho_{\rm max}/2$ and $\alpha(\hat\cM) \ge \alpha_{\rm max}/2$. 
(As before, what matters is that the regularity of $\hat\cM$ is controlled as a function of $\cM$.)

With the interpolating surface $\hat\cM$ defined, we have two choices:
\begin{itemize}
\item Because $\hat\cM$ is in the class $\bbM$, it comes\footnote{We are again invoking the axiom of choice here.} with a compact and convex domain with non-empty interior in $\bbR^k$, say $\hat\cU$, and a $C^2$ isometry defined on an open set containing $\hat\cU$ such that $\hat\phi(\hat\cU) = \hat\cM$. By applying the inverse of this isometry to the data points, we obtain an embedding in $\bbR^k$ given by 
\begin{equation} \label{embed_CS}
\hat u_1 := \hat\phi^{-1}(x_1), \dots, \hat u_n := \hat\phi^{-1}(x_n).
\end{equation}
\item We estimate the metric on $\cM$ by the metric on $\hat\cM$ as done in \eqref{d_hat} and then apply Classical Scaling to the set of estimated distances $(\d_{\hat\cM}(x_i, x_j))$ to get an embedding, $\hat u_1, \dots, \hat u_n \in \bbR^k$. 
\end{itemize}
The second option is seemingly more constructive, but it builds on the selection of $\hat\cM$, which is non-constructive. As it turns out, the two options give the same embedding (up to a rigid transformation). This is so because
\[\hat d_{ij} = \d_{\hat\cM}(x_i, x_j) = \|\hat\phi^{-1}(x_i) - \hat\phi^{-1}(x_j)\|,\]
so that $\hat\phi^{-1}(x_1), \dots, \hat\phi^{-1}(x_n)$ is a perfect realization of $(\hat d_{ij})$ into $\bbR^k$ and it is well-known that Classical Scaling returns a perfect realization when one exists.

\begin{theorem} 
\label{thm:isomap_approximation} 
There is a constant $C>0$ depending on $\cM$ such that, if $\eps \le 1/C$, the embedding \eqref{embed_CS} satisfies the error bound \eqref{embed_error}.
\end{theorem}

We used Classical Scaling above as this is the method used in the main Isomap variant~\cite{tenenbaum2000global}, but many other methods for MDS are available. For technical reasons which we explain later on, we use a landmark variant of Classical Scaling. This method was proposed by some of the same others~\cite{de2004sparse,silva2002global} as a speedup of Classical Scaling. It consists in 
1) selecting a few items; 
2) embedding these items by Classical Scaling; 
3) embedding the other items by lateration by reference to the points obtained in Step 2. 
Lateration consists in locating a point based on its distance to known `landmark' points. The lateration method used in~\cite{de2004sparse,silva2002global} was first proposed by Gower~\cite{gower1968adding}. It is known that, just like Classical Scaling, landmark Classical Scaling returns a perfect realization when one exists as long as the landmark items are chosen in Step 1 correspond to points that span the entire Euclidean space where the embedding takes place. Therefore, applying Classical Scaling or its landmark variant to $(\d_{\hat\cM}(x_i, x_j))$, we obtain an embedding of the form \eqref{embed_CS} in either case.
For reference, Classical Scaling is Algortihm~2 and Gower's lateration is Algorithm~3 in~\cite{arias2020perturbation}.

The technical reason why we use landmark Classical Scaling below is because of the perturbation bounds available to us.
For $y_1, \dots, y_m \in \bbR^k$ and $\bY := \{y_1, \dots, y_m\}$, simultaneously seeing as an $n \times k$ matrix with row vectors $y_i$, define 
\begin{equation} \label{diam_width}
\diam_2(\bY) = 2 \|\bY\|/\sqrt{m}, \qquad 
\width_2(\bY) = 2 \|\bY^\ddag\|^{-1}/\sqrt{m},
\end{equation}
where $\|\cdot\|$ denote the operator norm and $\bY^\ddag$ is the Moore--Penrose pseudo-inverse of $\bY$. Equivalently, these are the largest and smallest singular values of $(2/\sqrt{m}) \bY$.

The following is a slight edit of~\cite[Cor~2]{arias2020perturbation}. 
\begin{lemma}
\label{lem:CS}
Consider $y_1, \dots, y_m \in \bbR^k$ with diameter $\zeta$ and width $\omega$ as defined in \eqref{diam_width}, and with pairwise distances denoted $\delta_{ij} := \|y_i - y_j\|$.  
For an arbitrary set of nonnegative numbers $(\lambda_{ij})$, let $\eta^2 = \max_{i,j} |\lambda_{ij}^2 - \delta_{ij}^2|$.
There is a constant $C > 0$ depending only on $k$ such that, if $\eta/\omega \le 1/C$, then Classical Scaling with input dissimilarities $\{\lambda_{ij}\}$ (and dimension $k$) returns a point set $z_1, \dots, z_m \in \bbR^k$ satisfying
\begin{equation} \label{CS}
\min_{Q \in \bbQ_k} \bigg[\frac1m \sum_{i \in [m]} \|z_i - Q(y_i)\|^2\bigg]^{1/2} \le C (\zeta/\omega^2)\, \eta^2.
\end{equation}
\end{lemma}

The following is a slightly different variant of~\cite[Cor~3]{arias2020perturbation}. 
\begin{lemma}
\label{lem:lateration}
Consider $y_1, \dots, y_m \in \bbR^k$ with diameter $\zeta$ and width $\omega$ as defined in \eqref{diam_width}.
For a point $y \in \bbR^k$, set $\delta_{i} = \|y - y_i\|$.
For another point set $z_1, \dots, z_m \in \bbR^k$ and an arbitrary set of nonnegative numbers $(\lambda_i)$, let $\xi^2 = \frac1m \sum_{i \in [m]} \|z_i - y_i\|^2$ and $\eta^2 = \max_{i} |\lambda_{i}^2 - \delta_{i}^2|$.
There is a constant $C > 0$ depending only on $k$ such that, if $\xi/\omega \le 1/C$, Gower's lateration with inputs $(z_1, \dots, z_m)$ and $(\lambda_{i})$ returns $z \in \bbR^k$ satisfying 
\begin{equation} \label{lateration}
\|z - y\|
\le (C/\omega) \big(\eta^2 + m \zeta\, \xi\big).
\end{equation}
\end{lemma}

The reason we do not deal directly with Classical Scaling is because, in our setting, the width as defined in \eqref{diam_width} cannot be controlled for the entire sample $\bX$, so that the bound \eqref{CS} is not directly useful to control the performance of Classical Scaling. 
Instead, we employ landmark Classical Scaling and prove the following more general result.

\begin{theorem} \label{thm:embed}
Starting from an estimate of the distances $(\hat d_{ij})$ satisfying \eqref{distance_error}, for each set of $k+1$ sample points, embed them by Classical Scaling and compute their width as in \eqref{diam_width}. Apply landmark Classical Scaling to $(\hat d_{ij})$ with the $(k+1)$-tuple that gives the largest width as landmarks. 
There is a constant $C>0$ depending on $\cM$ such that, if $\eps \le 1/C$, the resulting embedding satisfies the error bound \eqref{embed_error}.
\end{theorem}

\begin{proof}
We start with the landmark points.
First, we lower bound their width by a constant that only depends on $\cM$.
Let $\bA = \{a_1, \dots, a_{k+1}\} \subset \cU$ be such that its convex hull has maximum width among all $(k+1)$-tuples in $\cU$. Because $\bU$ is an $\eps$-covering of $\cU$, there are $u_{i_1}, \dots, u_{i_{k+1}}$ such that $\|a_j - u_{i_j}\| \le \eps$ for all $j \in [k+1]$.
Assume without loss of generality that $i_j = j$ for all $j \in [k+1]$.  
Define $\bL = \{u_1, \dots, u_{k+1}\}$. Given that $\width_2(\bA)$ and $\width_2(\bL)$ are the smallest singular values of $(1/\sqrt{k+1}) \bA$ and $(1/\sqrt{k+1}) \bL$, respectively, Weyl's inequality gives
\[\width_2(\bL) \ge \width_2(\bA) - \tfrac1{\sqrt{k+1}} \|\bA - \bL\|.\]
We then have
\[\|\bA - \bL\|^2
\le \|\bA - \bL\|_F^2
= \sum_{j \in [k+1]} \|a_j - u_j\|^2
\le (k+1) \eps^2,\]
so that
\[\omega := \width_2(\bL) \ge \width_2(\bA) - \eps.\]
By construction, $\width_2(\bA)$ is a constant of $\cM$, say $2/C_1$, and henceforth we require that $\eps \le 1/C_1$ so that the landmarks have width $\omega \ge 1/C_1$ when embedded without error.

We now embed the landmark (or base) points $\bL$ by Classical Scaling. Let $\hat u_1, \dots, \hat u_{k+1}$ denote the embedded points. 
The embedding cannot be perfect as we do not know the true distances, but only have access to estimates. 
(We are about to apply \lemref{CS} with $y_i \gets u_i$ so that $\delta_{ij} = \|u_i - u_j\| = \d_\cM(x_i,x_j)$, and $\lambda_{ij} \gets \hat d_{ij}$, and the resulting embedding is $z_i \gets \hat u_i$. We denote by $C_3$ the constant in that lemma.) 
Define
\begin{equation} \label{eta_bound}
\eta^2 := \max_{i,j \in [k+1]} | \hat d_{ij}^2 - \|u_i - u_j\|^2| \le C_2 \eps^2,
\end{equation}
by \eqref{distance_error} and the fact that $\cU$ has a diameter that is a constant of $\cM$,  we apply \lemref{CS} to get that, if $\eta/\omega \le 1/C_3$, which holds if $\eps \le 1/(C_1 C_2^{1/2} C_3)$, then
\begin{equation}
\min_{Q \in \bbQ_k} \bigg[\frac1{k+1} \sum_{i \in [k+1]} \|\hat u_i - Q(u_i)\|^2\bigg]^{1/2} \le C_3 (\diam_2(\bL)/\omega^2)\, \eta^2 \le C_4 \eps^2,
\end{equation}
again using the fact that $\omega \ge 1/C_1$ and $\diam_2(\bL) \le \diam(\cU)$, which is a constant of $\cM$.
Henceforth, we assume that $Q = {\rm id}$ without loss of generality so that
\begin{equation} \label{xi_bound}
\xi^2 := \frac1{k+1} \sum_{i \in [k+1]} \|\hat u_i - u_i\|^2 \le (C_4 \eps^2)^2.
\end{equation}

Finally, we embed the remaining points, $x_i, i > k+1$, one-by-one by lateration. This is done by reference to $\hat u_1, \dots, \hat u_{k+1}$ using the estimated distances $(\hat d_{ij})$.
Take $p$ in $\{k+2, \dots, n\}$ and consider embedding $x_p$.
(We are about to apply \lemref{lateration} with $y_i \gets u_i$ for $i \in [k+1]$ and $y = u_p$ for some $p > k+1$, so that $\delta_{i} = \|u_p - u_i\| = \d_\cM(x_p, x_i)$, and $z_i \gets \hat u_i$ and $\lambda_i = \hat d_{pi}$. We denote by $C_5$ the constant in that lemma.) 
If $\xi/\omega \le 1/C_5$, which in view of \eqref{xi_bound} holds if $\eps \le 1/(C_1^{1/4} C_4^{1/2} C_5^{1/2})$, then
\begin{equation}
\|\hat u_p - u_p\|
\le (C_5/\omega) \big(\eta^2 + (k+1) \zeta\, \xi\big) \le C_1C_5 \big(C_2 \eps^2 + (k+1) \diam(\cU) C_4 \eps^2) =: C_6 \eps^2.  
\end{equation}
using $\omega \ge 1/C_1$, \eqref{eta_bound}, \eqref{xi_bound}, and $\diam_2(\bL) \le \diam(\cU)$, which is a constant of $\cM$.
\end{proof}

\begin{remark}
The procedure described in \thmref{embed} would in principle require going all possible $(k+1)$-tuples, and there are too many of them (on the order of $O(n^{k+1})$) for this to be practical. In principle, a randomized version would do essentially as well. It would amount to examining a number $N$ of $(k+1)$-tuples and choosing the best among them in terms of width. Then, the error bound \eqref{embed_error} would hold, say with twice the constant there, with probability exponentially close to~1 as a function of $N$.
Another possibility is to subsample $\bX$ to obtain an $(2\eps, 1/2)$-net (see \secref{nets}) and embed it by Classical Scaling --- which turns out to be fine in that case. Once embedded, it is computationally much easier to select a $(k+1)$-tuple of points with good width, and these are used to embedding the remaining sample points by lateration.
\end{remark}

\subsection{Information bound}
The same example used in \secref{information} can also be used to establish an information bound showing that the stated performance bound established in \thmref{isomap_approximation} is best possible. Indeed, using some of the same notation, on the one hand, $\cM_1$ is isometric to $\cU_1 := [0,1]^k$, with corresponding embedded points $u_i^1 := (x_{i,1}, \dots, x_{i,k})$ when $x_i = (x_{i,1}, \dots, x_{i,k})$.
On the other hand, $\cM_2$ is isometric to $\cU_2 
:= [0,1]^{k-1} \times [0,\Lambda(\gamma_\eps)]$,  
with corresponding embedded points $u_i^2 
:= (x_{i,1}, \dots, x_{i,k-1}, \lambda_\eps(x_{i,k}))$, where $\lambda_\eps$ is defined in \eqref{lambda}.
These embeddings are obviously not the only possibilities, but any other ones would have to be obtained by rigid transformations of these, and these particular ones are closest in average squared distance.
See \lemref{procrustes} below.

We then proceed to lower bound the squared average distance between these two embeddings
\begin{align*} 
\max_{i \in [m]^k} \|u_i^1 - u_i^2\|^2
\ge \frac1{m^k} \sum_{i \in [m]^k} \|u_i^1 - u_i^2\|^2 
 & = \frac1{m^k} \sum_{i \in [m]^k} \big[x_{i,k} - \lambda_\eps(x_{i,k})\big]^2 \\ 
 & = \frac1m \sum_{j \in [m]} \big[(j-1)\eps - (j-1)\eta\big]^2                            \\ 
 & = \frac{1}m (\eps-\eta)^2 \sum_{j \in [m]} (j-1)^2                                         \\ 
 & = (\eps-\eta)^2 (2m-1)(m-1) 
 \ge C_1^2 \eps^4,
\end{align*}
using at the end the lower bound in \eqref{eta_bounds} and the fact that $\eps = 1/(m-1)$.

Based on what we know of the true $\cM$, it could be $\cM_1$ as easily as $\cM_2$, and therefore, for any embedding $\hat u_i$,
\begin{align*}
\max_{\cM \in \{\cM_1, \cM_2\}} \min_{Q \in \bbQ_k} \max_{i \in [n]} \|\hat u_i - Q(u_i)\| 
&= \max\Big\{\min_{Q \in \bbQ_k} \max_{i \in [n]} \|\hat u_i - Q(u^1_i)\|, \min_{Q \in \bbQ_k} \max_{i \in [n]} \|\hat u_i - Q(u^2_i)\|\Big\} \\
 & \ge \tfrac12\, \min_{Q \in \bbQ_k} \max_{i \in [n]} \|u^1_i - Q(u^2_i)\| \\
& = \tfrac12\, \max_{i \in [n]} \|u^1_i - u^2_i\| \\  
& \ge \tfrac12 C_1 \eps^2.
\end{align*}

We have thus established the following.

\begin{theorem}
For any embedding method $\hat u$, the following is true. 
For any $\eps > 0$, there is a surface $\cM$ satisfying \aspref{M_isomap} and a set of points $x_1, \dots, x_n$ belonging to $\cM$ dense enough that \eqref{eps} holds, such that 
\begin{equation}
\min_{Q \in \bbQ_k} \max_{i \in [n]} \|\hat u_i - Q(u_i)\| \ge C^{-1} \eps^2. 
\end{equation}
\end{theorem}

\begin{lemma} 
\label{lem:procrustes} 
Consider two sets of points, $u_i = i = (i_1, \dots, i_k)$ and $v_i := (\alpha_1 i_1, \dots, \alpha_k i_k)$ for $i \in [m]^k$ and some real numbers $\alpha_1, \dots, \alpha_k$. Then, regardless of $\alpha$, the best alignment of these points by a rigid transformation is achieved by the identity transformation.
\end{lemma}

\begin{proof} 
The optimization problem we are studying is 
\begin{equation*} 
    \min_{Q \in \bbQ_k} \sum_{i \in [m]^k} \|v_i - Q(u_i)\|^2 
    = \min_{R \in \bbO_k} \min_{r \in \bbR^k} \sum_{i \in [m]^k} \|v_i - r - R u_i\|^2, 
\end{equation*} 
where $\bbO_k$ is the class of orthogonal transformations of $\bbR^k$. Given $R$, the minimum over $r$ is achieved at the average of $v_i-R u_i$, which reduces the problem to
\begin{equation*} 
\min_{R \in \bbO_k} \sum_{i \in [m]^k} \|v_i-\bar v - R (u_i-\bar u)\|^2,
\end{equation*} 
where $\bar u := (\frac{m+1}2, \dots, \frac{m+1}2)$ and $\bar v := (\alpha_1 \frac{m+1}2, \dots, \alpha_k \frac{m+1}2)$ are the barycenters of $u_1, \dots, u_n$ and $v_1, \dots, v_n$, respectively. 
Let $U$ and $V$ be the matrices with row vectors $u_i - \bar u$ and $v_i - \bar v$, respectively. It is well-known that the optimal orthogonal transformation solving the problem the optimal $R$ above is $A B^\top$ if $A \Lambda B^\top$ is a singular value decomposition of $M := V^\top U$.
To show that this is the identity matrix, it suffices to show that $M$ is diagonal, or equivalently, that the canonical basis vectors of $\bbR^k$, denoted $e_1, \dots, e_k$ below, are eigenvectors for $M$. 
Take any $t \in [k]$. 
Then, noting that $M = \sum_i (v_i - \bar v) (u_i - \bar u)^\top$, we have for $s \ne t$, 
\begin{align*} 
    (M e_t)_s 
     & = \sum_{i \in [m]^k} (v_i - \bar v)_s (u_i - \bar u)                                 _t\\ 
     & = \sum_{i \in [m]^k} \alpha_s (i_s-\tfrac{m+1}2) (i_t-\tfrac{m+1}2)                        \\ 
     & = m^{k-2} \alpha_s \sum_{i_s \in [m]} (i_s-\tfrac{m+1}2) \sum_{i_t \in [m]} (i_t-\tfrac{m+1}2) 
    = 0, 
\end{align*} 
and, similarly,  
\begin{align*} 
    (M e_t)_t 
     & = m^{k-1} \alpha_t \sum_{i_t \in [m]} (i_t-\tfrac{m+1}2)^2 
    =: a\, \alpha_t, \quad a := m^k (m^2-1)/12.
\end{align*} 
Hence, $M e_t = a\, \alpha_t e_t$, so that $e_t$ is indeed an eigenvector of $M$ (for the eigenvalue $a \alpha_t$).
\end{proof}

\subsection{Mesh Isomap}

Isomap consists in 1) building a neighborhood graph; 2) computing all pairwise graph distances; 3) applying Classical Scaling to the resulting distances. Steps 1 and 2 have for purpose to estimate the pairwise intrinsic distances on the underlying surface, and this is where we bring an improvement, as we replace these steps with a more accurate way of estimating distances based on a mesh construction. Step 3 remains the same in principle, or it can be replaced by any other method for MDS as was done for Isomap, where landmark Classical Scaling was proposed as a faster alternative~\cite{silva2002global}.  See \algref{mesh_isomap}, where {\sf mesh} denotes a generic mesh construction algorithm and {\sf meshDistances} a generic algorithm for computing all pairwise distances between the vertices of a given mesh, and {\sf MDS} denotes a generic method for MDS.

\begin{algorithm}[!tpb] 
\caption{Mesh Isomap} 
\label{alg:mesh_isomap} 
\begin{algorithmic} 
    \STATE {\bf Input:} point set $x_1, \dots, x_n$ in $\bbR^d$, embedding dimension $k$, any parameter of {\sf mesh} 
    \STATE {\bf Output:} point set $\hat u_1, \dots, \hat u_n$ in $\bbR^k$ 
    \medskip 
    \STATE {\bf 1:} 
    Apply {\sf mesh} to $x_1, \dots, x_n$ to get a mesh $\hat{\sf M}$ 
    \STATE {\bf 2:} 
    Apply {\sf meshDistances} to $\hat{\sf M}$ to get a matrix of pairwise distances $\hat D$ 
    \STATE {\bf 3:} 
    Apply {\sf MDS} to $\hat D$ to produce a point set $u_1, \dots, u_n$ in $\bbR^k$ 
    \STATE {\bf Return:} the point set $\hat u_1, \dots, \hat u_n$ 
\end{algorithmic}
\end{algorithm}

\begin{remark}\label{rem:isomap_discussion}
Although the method as such seems new, it was mentioned in~\cite{balasubramanian2002isomap} in a discussion of the original Isomap paper~\cite{tenenbaum2000global}. In that discussion, the authors mention previous work of theirs~\cite{schwartz1989numerical} on the flattening of a mesh, which consists in computing the distances on the mesh and then applying the multidimensional scaling method of~\cite{sammon1969nonlinear}. Note however that the setting is different in that a mesh is assumed to be provided, while we only assume that a point cloud is provided. Although this distinction was immediately underscored by the authors of Isomap in their rebuttal, they also failed to realize that a better performance could be gained by using a mesh construction in the process of computing the pairwise distances. This is the main novelty in \algref{mesh_isomap}.
\end{remark}

In an effort to obtain a performance bound for Mesh Isomap, we specialize the algorithm by using as mesh construction the tangential Delaunay complex corrected for inconsistencies based on estimated tangent spaces described in \secref{TDC} and using as method for MDS  landmark Classical Scaling as described in \secref{embedding}. 

We established in \thmref{TDC_bound} that the mesh construction yields distance estimates that satisfy \eqref{distance_error}, but we did so under the assumption that the surface $\cM$ does not have a boundary. It turns out that the construction is local in that the computation of a given simplex in the complex only depends on the sample points that are within $C_1\eps$ of the simplex~\cite[Lem~8.10(3)]{Boissonnat2018}; and the estimation of the tangent spaces at a given point, as carried out in \secref{tangent_estimation}, only depends on the sample points that are within $C_2\eps$ of the point of interest. 
This leads us to anticipate that the embedding error \eqref{embed_error} applies here as well --- even though $\cM$ has a boundary --- at least for data points that are $C_3 \eps$ away from $\partial \cM$. 
This points to the possibility that {\em this variant of Mesh Isomap is minimax rate-optimal for manifold learning in the situation where the submanifold is isometric to a convex domain}.

\subsection{Numerical experiments}

\begin{figure}[htbp!] 
\centering 
\begin{subfigure}[t]{.475\linewidth} 
    \centering 
    \includegraphics[width=1.1\linewidth, trim=60 30 40 40, clip]{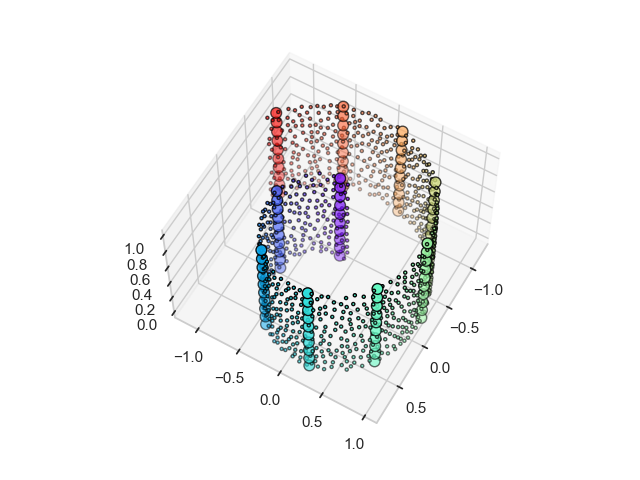} 
    \caption[]{{\small The large colored points on the Swiss roll (of sample size $n=1000$) are the 
                landmarks that will be embedded in $\bbR^2$ using Classical Scaling.}} 
\end{subfigure} 
\quad 
\begin{subfigure}[t]{.475\linewidth} 
    \centering 
    \includegraphics[width=\linewidth, trim=30 0 50 35, clip ]{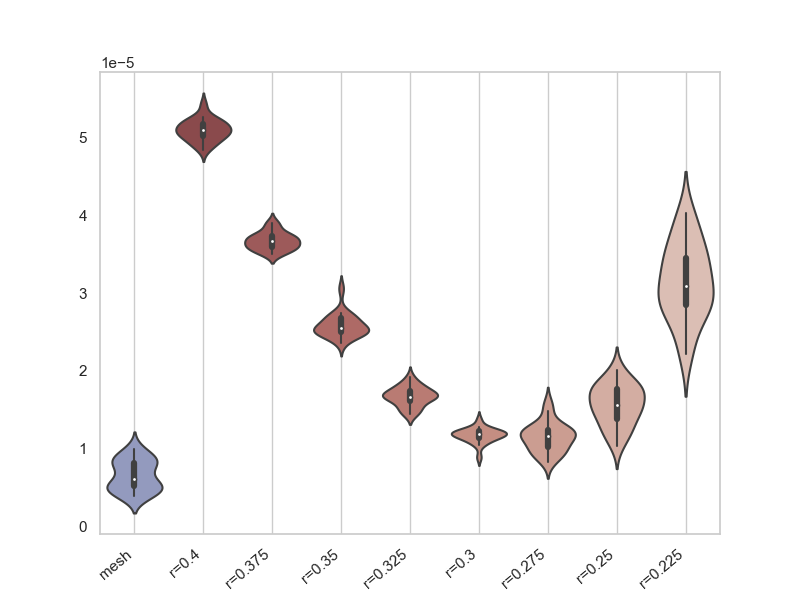} 
    \caption[]{{\small Procrustes error of the embeddings returned by Mesh Isomap (blue) and regular Isomap with varying connectivity radius. The graph approximation 
                provides the most accurate embedding with $r=0.3$, but the mesh approximation is more accurate still. Based on 20 repeats.}} 
\end{subfigure} 
\vskip\baselineskip 
\begin{subfigure}[t]{.475\linewidth} 
    \centering 
    \includegraphics[width=\linewidth, trim=20 20 40 40, clip]{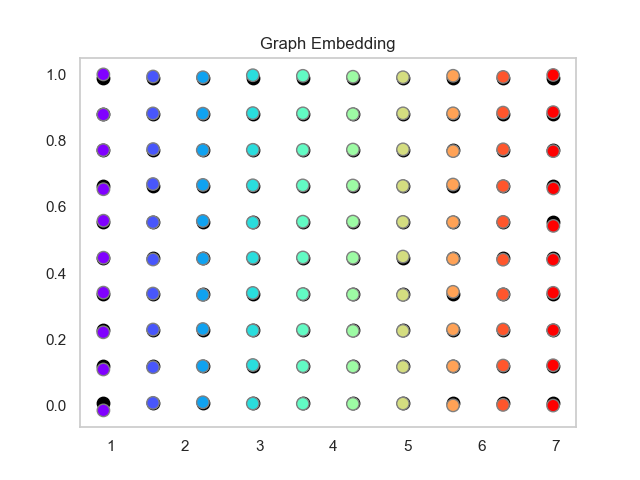} 
    \caption[]{\small Typical output of regular Isomap with the best choice of neighborhood radius $r=0.3$. 
        The original points are in black and the output is in color.} 
\end{subfigure} 
\quad 
\begin{subfigure}[t]{.475\linewidth} 
    \centering 
    \includegraphics[width=\linewidth, trim=20 20 40 40, clip]{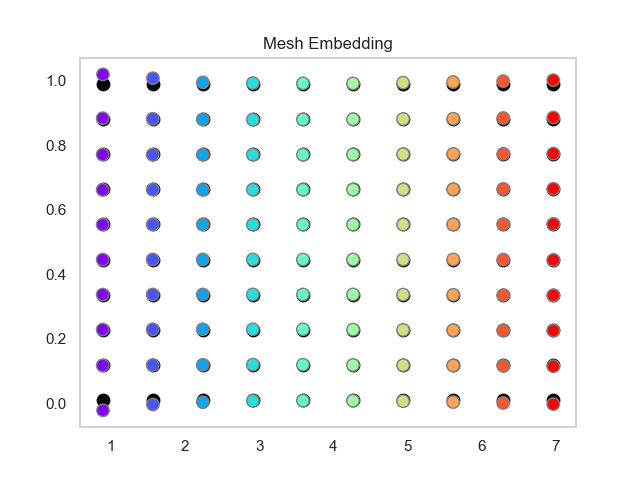} 
    \caption[]{\small Typical output of Mesh Isomap.             The original points are in black and the output is in color.} 
\end{subfigure} 
\caption{Comparison of Isomap and Mesh Isomap on the Swiss roll.} 
\label{fig:exp_isomap}
\end{figure}

In this subsection, we compare the original Isomap algorithm with Mesh Isomap in simulations. We do so on the Swiss roll, which is perhaps the most emblematic surface in manifold learning. In our implementation of Mesh Isomap we used the same tangential Delaunay complex construction~\cite{gudhi:TangentialComplex} as we did in \secref{numerics}.
The embedding error was computed up to a rigid transformation by Procrustes.

The result of this experiment is reported in \figref{exp_isomap}.
As in \secref{numerics}, the average performance of Mesh Isomap is noticeably better than that of regular Isomap across all choices of the connectivity radius.

\section{Discussion}
\label{sec:discussion}

In two places in the paper, we suspected but were not able to prove that surfaces were $O(\eps^2)$-distortions of each other --- and had to use a different route to get to the desired result. 
 
In \secref{distortion}, we believe that $\hat\cM$ and $\cM$ are $O(\eps^2)$-distortions of each other. If this had been established, then it would have enabled us to apply \corref{distortion} to immediately get \thmref{approximation}. As we were not able to prove this claim, we used a different route through \lemref{both_projections} instead. 

\begin{conjecture}
There are universal constants $C_1, C_2>0$ such that, if $\cM$ and $\cS$ are compact and connected $k$-dimensional submanifolds without boundary with reach $\ge \rho$, and if they are within Hausdorff distance $h \le \rho/C_1$ of each other, then they are $(C_2 h/\rho)$-distortions of each other.
\end{conjecture}

In \secref{meshes}, we believe that $\hat\cT_{\rm c}$ and $\cM$ are $O(\eps^2)$-distortions of each other. 
(In \cite[Th~4.1]{aamari2018stability} and its proof via \cite[Lem~4.2]{aamari2018stability}, we see that $\cM$ and $\tilde\cM$ are $O(\eps)$-distortions of each other. It would have been enough to have $O(\eps^2)$ in place of $O(\eps)$.)
If this had been established, then it would have enabled us to apply \corref{distortion} to immediately get \thmref{TDC_bound}. Although the proof of that result is short, we could have avoided the use of multiple net construction as described in \secref{nets} --- see \remref{nets} there. It would have been enough to work with a single net (obtained by subsampling $\bX$) and then the error bound \eqref{distance_error} would have been established for all sample points (including those outside the net) by way of \corref{distortion}. 

\begin{conjecture}
There are universal constants $C_1, C_2>0$ such that the following holds. Suppose $\cM$ is a compact and connected $k$-dimensional submanifold without boundary with reach $\ge \rho$. Consider a $k$-simplicial complex $\cT$ with vertices on $\cM$ that is homeomorphic to $\cM$ and such that all its $k$-simplexes have diameter $\le h$ with $h/\rho \le 1/C_1$ and thickness $\ge 1/C_1$. Then they are $(C_2 h/\rho)$-distortions of each other.
\end{conjecture}

\section*{Acknowledgments}
When EAC presented prior work on this topic at the {\em 6th Princeton Day of Statistics}, Amit Moscovich proposed this idea of using an approximation to the underlying surface to possibly obtain a better approximation rate, which is at the foundation of the present paper. Although we later discovered that this idea had been entertained earlier (see \remref{isomap_discussion}), we are nonetheless indebted to him as this idea got us started on this project.
We are also grateful to Eddie Aamari, Jean-Daniel Boissonnat, Fr\'ed\'eric Chazal, and Justin Roberts for helpful discussions and pointers to the literature.
This work was partially supported by the US National Science Foundation (DMS 1916071).

\small
\bibliographystyle{abbrv}
\bibliography{biblio.bib}

\end{document}

%% file: notation.tex
\def\d{\mathsf{d}}

\usepackage{scalerel}

\newtheorem{theorem}{Theorem}[section]
\newtheorem{proposition}[theorem]{Proposition}
\newtheorem{lemma}[theorem]{Lemma}
\newtheorem{corollary}[theorem]{Corollary}

\theoremstyle{definition}

\theoremstyle{remark}
\newtheorem{remark}[theorem]{Remark}
\newtheorem{assumption}[theorem]{Assumption}
\newtheorem{conjecture}[theorem]{Conjecture}

\numberwithin{equation}{section}
\numberwithin{figure}{section}

\def\beq{\begin{equation}} 
\def\eeq{\end{equation}}
\def\beqn{\begin{eqnarray*}}
\def\eeqn{\end{eqnarray*}}
\def\Bitem{\begin{itemize}\setlength{\itemsep}{.2in}}
\def\bitem{\begin{itemize}\setlength{\itemsep}{.05in}}
\def\eitem{\end{itemize}}
\def\Benum{\begin{enumerate}\setlength{\itemsep}{.2in}}
\def\benum{\begin{enumerate}\setlength{\itemsep}{.05in}}
\def\eenum{\end{enumerate}}
\def\bmult{\begin{multline*}}
\def\emult{\end{multline*}}
\def\bcenter{\begin{center}}
\def\ecenter{\end{center}}
\def\bframe{\begin{frame}}
\def\eframe{\end{frame}}

\newcommand{\thmref}[1]{Theorem~\ref{thm:#1}}
\newcommand{\prpref}[1]{Proposition~\ref{prp:#1}}
\newcommand{\corref}[1]{Corollary~\ref{cor:#1}}
\newcommand{\lemref}[1]{Lemma~\ref{lem:#1}}
\newcommand{\secref}[1]{Section~\ref{sec:#1}}
\newcommand{\figref}[1]{Figure~\ref{fig:#1}}
\newcommand{\tabref}[1]{Table~\ref{tab:#1}}
\newcommand{\algref}[1]{Algorithm~\ref{alg:#1}}
\newcommand{\remref}[1]{Remark~\ref{rem:#1}}

\newcommand{\aspref}[1]{Assumption~\ref{asp:#1}}

\DeclareMathOperator{\diam}{diam}
\DeclareMathOperator{\width}{width}
\DeclareMathOperator{\dist}{dist}

\def\cG{\mathcal{G}}

\def\cM{\mathcal{M}}

\def\cS{\mathcal{S}}
\def\cT{\mathcal{T}}
\def\cU{\mathcal{U}}

\def\bA{\mathbf{A}}

\def\bL{\mathbf{L}}

\def\bU{\mathbf{U}}

\def\bX{\mathbf{X}}
\def\bY{\mathbf{Y}}

\def\bbM{\mathbb{M}}

\def\bbO{\mathbb{O}}

\def\bbQ{\mathbb{Q}}
\def\bbR{\mathbb{R}}

\def\eps{\varepsilon}

\def\1{\mathbbm{1}}
